\documentclass[journal]{IEEEtran}
\usepackage{cite}

\ifCLASSINFOpdf
  \usepackage[pdftex]{graphicx}
  \graphicspath{{fcv2016sonoyamaMI/images/}}
\else
\fi
\usepackage{amsmath}
\usepackage{amssymb}
\usepackage{amsthm}
\usepackage[all, warning]{onlyamsmath}
\usepackage{bm}

\usepackage{url}

\usepackage{hyperref}

\newtheorem{problem}{Problem}

\newtheorem{define}{Definition}
\newtheorem{corollary}{Corollary}
\newtheorem{lemma}{Lemma}

\def\a{{\bm{a}}}
\def\b{{\bm{b}}}
\def\bb{\tilde{\b}}
\def\c{{\bm{c}}}
\def\d{{\bm{d}}}

\def\q{{\bm{q}}}

\def\x{{\bm{x}}}
\def\z{{\bm{z}}}
\def\w{{\bm{w}}}
\def\th{{\boldsymbol{\theta}}}
\def\bphi{{\boldsymbol{\phi}}}
\def\bxi{{\boldsymbol{\xi}}}
\def\bmu{{\boldsymbol{\mu}}}
\def\vec{{\mathrm{vec}}}

\begin{document}
%
\title{Domain Adaptation with $L_2$ constraints
for classifying images from different endoscope systems}
%
%
%

\author{Toru~Tamaki,
Shoji~Sonoyama,
Takio~Kurita,
Tsubasa~Hirakawa,
Bisser~Raytchev,
Kazufumi~Kaneda,
Tetsushi~Koide,
Shigeto~Yoshida,
Hiroshi~Mieno,
Shinji~Tanaka,
and~Kazuaki~Chayama
}

%
%

\markboth{Journal of \LaTeX\ Class Files,~Vol.~14, No.~8, August~2015}%
{Shell \MakeLowercase{\textit{et al.}}: Bare Demo of IEEEtran.cls for IEEE Journals}
%



\maketitle

\begin{abstract}
This paper proposes a method for domain adaptation
that extends the maximum margin domain transfer (MMDT) proposed by
Hoffman et al., by introducing $L_2$ distance constraints between
samples of different domains; thus, our method is denoted as MMDTL2.
Motivated by the differences between the images taken by
narrow band imaging (NBI) endoscopic devices,
we utilize different NBI devices as different domains
and estimate the transformations between samples of different domains,
i.e., image samples taken by different NBI endoscope systems.
We first formulate the problem in the primal form,
and then derive the dual form with much lesser computational costs as compared to the naive approach. 
From our experimental results using NBI image datasets from two different NBI endoscopic devices, 
we find that MMDTL2 is better than MMDT and also support vector machines without adaptation, especially
when NBI image features are high-dimensional and the per-class training samples are greater than 20.
\end{abstract}

\begin{IEEEkeywords}
Domain adaptation; Dual formulation; Kernels; NBI endoscopy
\end{IEEEkeywords}

%
\IEEEpeerreviewmaketitle

\section{Introduction}

In many hospitals, endoscopic examinations (i.e., colonoscopies) using narrow band imaging (NBI) systems are
widely performed to diagnose colorectal cancer \cite{Tanaka2006},
which is a major cause of cancer deaths worldwide \cite{CancerResearchUK2}.
During examinations, endoscopists observe and examine a polyp 
based on its visual appearance, including via NBI magnification findings \cite{Kanao2009,Oba2010}, as shown in Figure \ref{fig:nbi_magnification}.
To support proper diagnosis during examinations,
a computer-aided diagnostic system based on the textural appearance of polyps would be helpful;
thus, numerous patch-based classification methods for endoscopic images have been proposed
\cite{Hafner2010a,Hafner2010b,Kwitt2010,Gross2009,Stehle2009,Tischendorf2010,Tamaki2013}.

This paper focuses on the inconsistencies between training and test images.
As with other frequently used machine learning approaches, training classifiers assumes that
the distribution of features extracted from both training and test image datasets are the same;
however, different endoscope systems may be used to collect training and test datasets, causing such an assumption to be violated.
Further, given the rapid development of medical devices (i.e., endoscopies in this case),
hospitals can introduce new endoscopes after training images have already been taken.
In addition, classifiers may be trained with a training dataset collected by a certain type of endoscope in one hospital,
while another hospital might use the same classifiers for images taken by a different endoscope.
In general, such inconsistencies lead to the deterioration of classification performance;
hence, collecting new images for a new training dataset may be necessary or is at least preferred.
However, this is not the case with medical images. It is impractical to collect enough sets of images for all types and manufacturers of endoscopes.
Active learning \cite{Settles2012} deals with small samples;
however, this method is not helpful because it selects small samples from unlabeled large training samples.

Figure \ref{fig:appearance} shows an example of differences between textures captured by different endoscope systems.
More specifically, the images shown in Figures \ref{fig:appearance}(a) and \ref{fig:appearance}(b) are the same scene from a printed sheet of a colorectal polyp image
taken by different endoscope systems at approximately the same distance to the sheet from the endoscopes.
Even for the same manufacture (e.g., Olympus) and the same modality (e.g., NBI),
images may differ in terms of factors such as resolution, image quality, sharpness, brightness, and viewing angle.
These differences may impact classification performance.

To address this problem,
Sonoyama et al. \cite{Sonoyama2015a} proposed a method based on transfer learning \cite{pan2010survey,raina2007self,dai2007boosting,silver2005nips}
to estimate a transformation matrix between feature vectors of training and test datasets captured by different (i.e., old and new) devices.
In this prior study, we formulated the problem as a constraint optimization problem and developed an algorithm to estimate a linear transformation;
however, a key limitation is that corresponding datasets are required, i.e., each test image (i.e., taken by a new device) must have a corresponding training image (i.e., taken by an old device). Further, these images must capture the same polyp to properly estimate the linear transformation.
These restrictions are rather strong, causing our system to be somewhat impractical.

\begin{figure}[t]
  \centering
  \includegraphics[width=\linewidth]{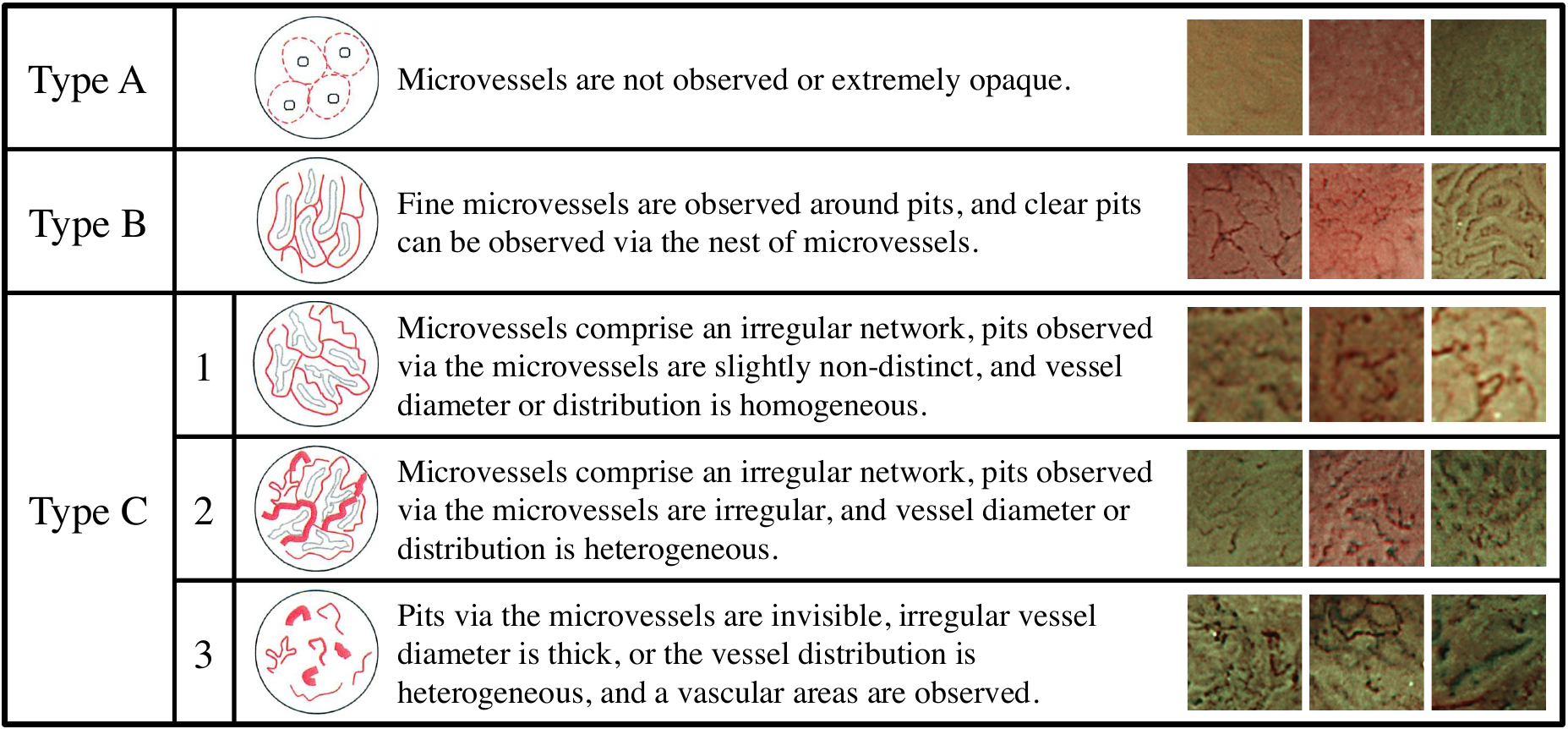}
  \caption{NBI magnification findings \cite{Kanao2009}.}
  \label{fig:nbi_magnification}
\end{figure}

\begin{figure}[t]
 \begin{minipage}{0.45\linewidth}
  \centering
   \includegraphics[width=\linewidth]{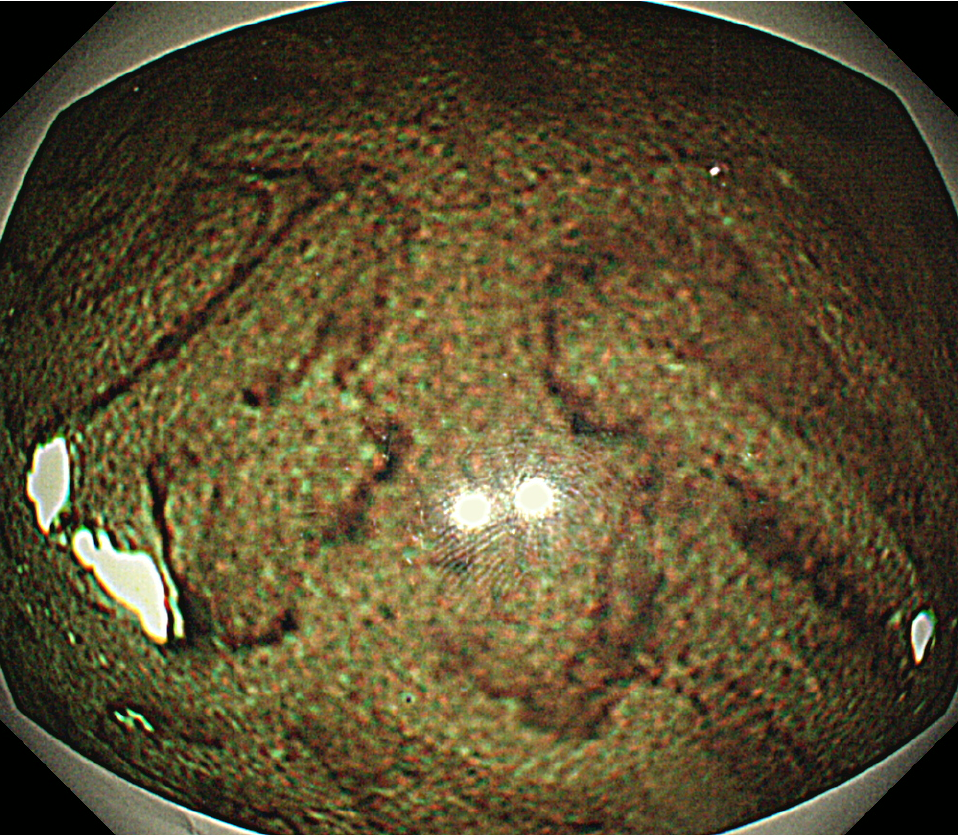}
  (a)\\
 \end{minipage}
 \hfill
 \begin{minipage}{0.45\linewidth}
  \centering
   \includegraphics[width=\linewidth]{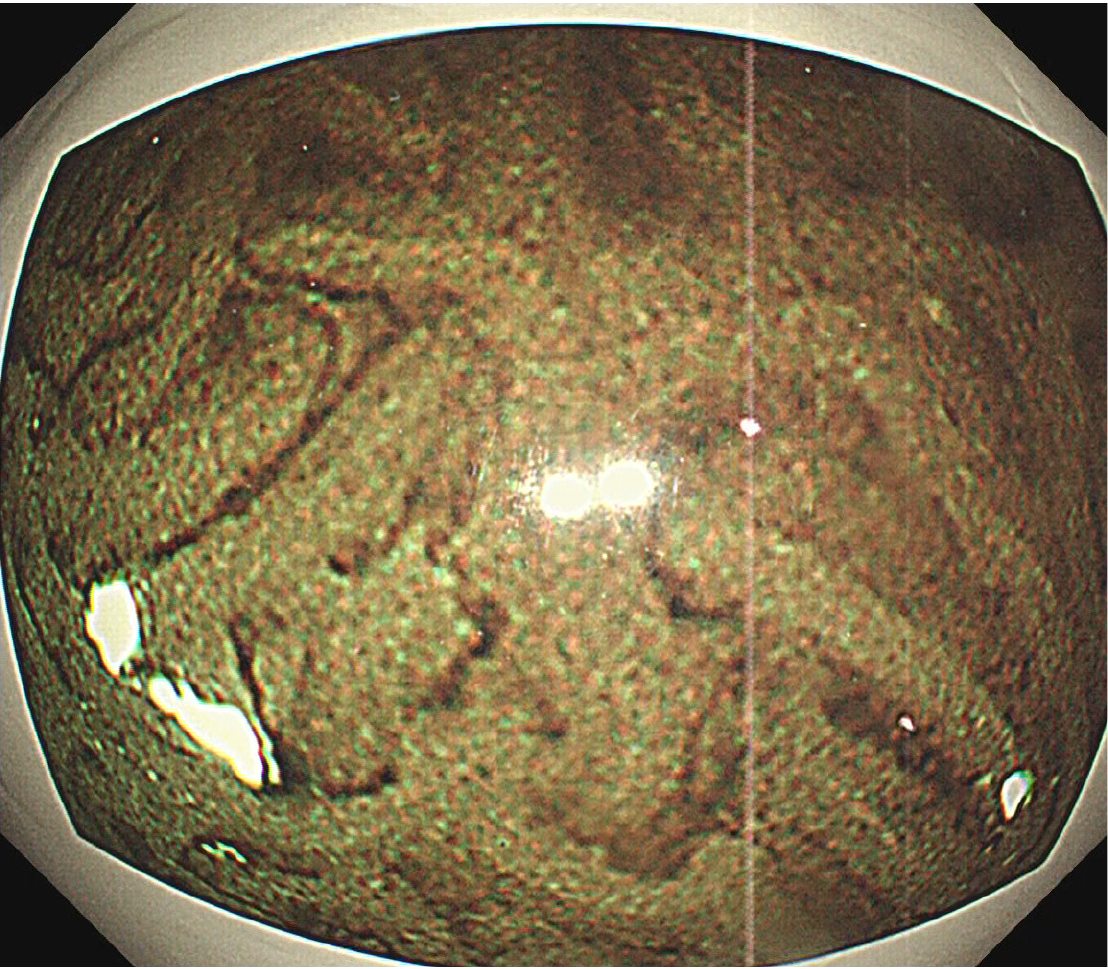}
  (b)\\
 \end{minipage}
 \caption{Example showing differences in appearance of images captured by different endoscope systems:
(a) an image taken by an older system
(i.e., video system center: OLYMPUS EVIS LUCERA CV-260 \cite{CV-260};
endoscope: OLYMPUS EVIS LUCERA COLONOVIDEOSCOPE CF-H260AZI \cite{CF-H260AZI});
(b) an image of the same scene taken by a newer system
(i.e., video system center: OLYMPUS EVIS LUCERA ELITE CV-290 \cite{CV-290};
endoscope: OLYMPUS CF-H260AZI \cite{CF-H260AZI}).
}
 \label{fig:appearance}
\end{figure}

Therefore, this paper proposes an improved method for a task that does not require image-by-image correspondences between training and test datasets.
More specifically, 
we extend the transfer learning method proposed by Hoffman et al. \cite{Hoffman2013a,Hoffman2014a}, called maximum margin domain transfer (MMDT).
Their approach was a domain adaptation technique to handle the domain shift problem, in which
distributions of classes (or categories) in one domain, called the {\em source}, change in another domain, called the {\em target}. This situation occurs in various applications, and hence, domain adaptation and transfer learning have already been widely studied.

Compared to previous studies, MMDT had the following advantages: (1) applicability to multiclass problems with one-vs-rest support vector machines (SVMs); (2) the ability to use different feature dimensions in the source and target domains; (3) the ability to handle unlabeled target samples by estimating global class-independent affine transformation matrix $W$; and (4) scalability to the number of constraints, i.e., MMDT solves $K L_s$ constraints as compared to $L_t L_s$ in the previous studies,
where $K$ is the number of classes and $L_s$ and $L_t$ are the feature dimensions of the source and target domains.

In this paper%
\footnote{A conference version of this paper was presented \cite{Sonoyama2016a}.
This paper extends that version with more rigorous derivations,
the compact form of the dual formulation,
and the kernelization of the method, as well as experiments from different aspects.},
we therefore propose a non-trivial extension to MMDT
for handling inconsistencies between NBI endoscopic devices.
We summarize the contributions of this paper as follows.
\begin{itemize}
\item
First, we add $L_2$ distance constraints to MMDT; our method is thus called MMDTL2. The original formulation of MMDT uses the Frobenius norm of transformation $W$ as a regularizer; however, pulling transformation $W$ into a zero matrix is not intuitive and might not have a good effect on transferring samples. Other regularizers were discussed by \cite{Hoffman2014a}, e.g., an identity matrix when $L_s = L_t$, but no examples were given for cases of $L_s \neq L_t$. For the latter cases, the target samples in one category should be transformed into the same category as that of the source domain. Therefore, we propose using the $L_2$ distances between the source and transformed target samples to regularize $W$.
\item
Second, we explicitly formulate MMDTL2 as a quadratic programming (QP) problem. In \cite{Hoffman2013a,Hoffman2014a}, the MMDT problem was described, but not in QP form. In this paper, we explicitly formulate MMDTL2 in the standard 
primal QP form, which includes MMDT as a special case (i.e., where no $L_2$ constraints are used).
\item
Third, we derive the dual QP problem of MMDTL2, where the QP form mentioned above is the primal form. 
The computational costs of MMDT and the primal QP form of MMDTL2 can be very large for samples with large dimensions because:
\begin{itemize}
  \item for MMDT, affine transformation matrix $W$, which is of size $L_s \times L_t$ (where $L_s$ and $L_t$ are the feature dimensions of the source and target domains), can be very large;
  \item for the primal QP form of MMDTL2, a matrix of size $L_s (L_t + 1) \times L_s (L_t + 1)$ appears in the computation, which is much larger than $W$.
\end{itemize}
In contrast, our derived dual QP form of MMDTL2 is more scalable because it involves a $KM \times KM$ matrix, 
where $M$ is the number of target samples and $K$ is the number of classes. 
Typically, $M$ is limited or much smaller than the number of source samples, which is common and therefore reasonable for domain adaptation problems, including our problem of NBI device inconsistency.
\item
Finally, we show that the primal QP solution can be converted from the dual QP solution with much lower computational cost. With our derived formula, this conversion needs matrices of size $KM \times KM$ at most, while without elaboration, a matrix of size $L_s (L_t + 1) \times L_s (L_t + 1)$ appears in the conversion.
\item
In addition, we derive a kernelization of the dual QP formulation, which enables us to use a nonlinear transformation as $W$ for better performance.
\end{itemize}

Note that our dual form is different from the dual form derived by Rodner et al. \cite{Rodner2013}. Their motivation was to make MMDT scalable in terms of the number of target samples, because in their study, they attempted to adapt large datasets such as ImageNet. Therefore, their dual form still suffers from the large feature dimensions of the source and target samples. In contrast,
our formulation is scalable in terms of feature dimensions.

The rest of the paper is organized as follows. We formulate problems of MMDT and MMDTL2 in Section 2, and then derive the primal form in Section 3. In Section 4, we show the dual form, and in Section 5, we obtain the primal solution from the dual solution. In Section 6, we present our experimental results using datasets of actual NBI endoscopic images. Finally, in Section 7, we conclude this paper and provide avenues for future work.

\section{Problem formulation}

In this section, we introduce the problems of MMDT \cite{Hoffman2014a}
and our proposed MMDTL2.

\begin{problem}[MMDT]\label{prob:mmdt}

Suppose we are given a training set
$\chi^s = \{ \x_n^s, y_n^s \}_{n=1}^N \subset \mathbb{R}^{L_s}\times\{ 1,2,\ldots,K \}$
in the source domain and another set
$\chi^t = \{ \x_m^t, y_m^t \}_{m=1}^M \subset \mathbb{R}^{L_t}\times\{ 1,2,\ldots,K \}$
in the target domain for a $K$-class classification problem.

MMDT solves the following optimization problem:
\begin{align}
\min_{W, \hat\Theta}
\frac{1}{2} \| W \|_F^2 +
\frac{1}{2} \| \hat\Theta \|_F^2 +
c_t \mathcal{L} (W, \hat\Theta, \chi^t ) +
c_s \mathcal{L} (\hat\Theta, \chi^s ),
\end{align}
where $c_t$ and $c_s$ are weights and
\begin{align}
\mathcal{L} (W, \hat\Theta, \chi^t ) 
&=
\sum_{k,m}
\max \left(
0, 1 - y_{km}^t \hat\th_k^T 
  \begin{pmatrix}
  W \hat\x_m^t \\
  1
  \end{pmatrix}
   \right)
\end{align}
and
\begin{align}
\mathcal{L} (\hat\Theta, \chi^s )
&=
\sum_{k,n}
\max \left(
0, 1 - y_{kn}^s \hat\th_k^T 
   \hat\x_n^s
   \right)
\end{align}
are hinge loss functions.
Here, $\hat\th_k \in \mathbb{R}^{L_s+1}$ is an SVM hyperplane parameter
(including weights $\th_k \in \mathbb{R}^{L_s}$ and bias $b_k \in \mathbb{R}$,
i.e., $\hat\th_k = (\th_k^T, b_k )^T)$ for the $k$th class
stacked into a matrix $\hat\Theta = (\hat\th_1, \hat\th_2, \ldots, \hat\th_K) $,
$y_{km}^t = 2\delta(y_m^t, k)-1 \in \{-1,1\}$ is a label in terms of the $k$th hyperplane, and $c_t$ and $c_s$ are weights.

Note that $\hat{\x}$ denotes an augmented vector with 1 as its last element, i.e.,
$\hat\x = 
\begin{pmatrix}
\x \\ 1
\end{pmatrix}
$.

\end{problem}

In this problem, we simultaneously obtain $K$ SVM classifiers and transformation $W \in \mathbb{R}^{L_s \times L_t}$. One-vs-all SVMs are used for multiclass classification; thus, $K$ hyperplane parameters $\hat\th_k$ are obtained in the source domain. Target samples $\x^t_m$ are transformed by $W$ from the target domain to the source domain, and then the loss function causes them to be classified by the SVMs.

Because this problem is non-convex, an alternating optimization approach was used in \cite{Hoffman2014a}.

\begin{problem}[MMDT with iteration]\label{prob:mmdt_alt}

MMDT solves problem \ref{prob:mmdt} by iteratively solving subproblems
\begin{align}
&
\min_{W}
\frac{1}{2} \| W \|_F^2 +
c_t \mathcal{L} (W, \hat\Theta, \chi^t ),
\end{align}
and
\begin{align}
& \min_{\hat\Theta}
\frac{1}{2} \| \hat\Theta \|_F^2 +
c_t \mathcal{L} (W, \hat\Theta, \chi^t ) +
c_s \mathcal{L} (\hat\Theta, \chi^s ),
\label{eq:svm_subproblem_mmdt}
\end{align}
initializing $\hat\Theta$ with
\begin{align}
\mathrm{arg}\min_{\hat\Theta}
\frac{1}{2} \| \hat\Theta \|_F^2 +
c_s \mathcal{L} (\hat\Theta, \chi^s ).
\end{align}

\end{problem}

As noted in the Introduction, the use of Frobenius norm $\| W \|_F^2$
is not intuitive for the transformation matrix. Further, it might not be a good choice for small values of $c_t$ because the obtained solution is pulled toward a zero matrix; however, the use of large values of $c_t$ impacts the SVM subproblem (\ref{eq:svm_subproblem_mmdt}) because C-SVM solvers are known to be unstable for large values of parameter $C$.

In this paper, we therefore propose the following problem, which we call MMDTL2. MMDTL2 incorporates additional constraints of $L_2$ distances to pull target samples to source samples of the same category.

\begin{problem}[MMDTL2]\label{prob:mmdtL2}

MMDT with $L_2$ constraints solves problem \ref{prob:mmdt} by iteratively solving subproblems
\begin{align}
&
\min_{W}
\frac{1}{2} c_f \| W \|_F^2 +
c_t \mathcal{L} (W, \hat\Theta, \chi^t ) +
c_d \mathcal{L} (W, \chi^s, \chi^t ),
\label{eq:w_subproblem_mmdtl2}
\end{align}
and
\begin{align}
& \min_{\hat\Theta}
\frac{1}{2} \| \hat\Theta \|_F^2 +
c_t \mathcal{L} (W, \hat\Theta, \chi^t ) +
c_s \mathcal{L} (\hat\Theta, \chi^s )
\label{eq:svm_subproblem_mmdtl2}
\end{align}
with the same initialization as problem \ref{prob:mmdt_alt},
where
\begin{align}
c_d \mathcal{L} (W, \chi^s, \chi^t )
= \frac{1}{2}   \sum_{m=1}^M \sum_{n=1}^N s_{nm} \| W \hat\x_m^t - \x_n^s \|_2^2,
\end{align}
where $s_{nm} = c_d \delta(y_n^s, y_m^t)$ is a weight between samples $\x_m^t\in \mathbb{R}^{L_t}$ and $\x_n^s \in \mathbb{R}^{L_s}$
and $c_d$ is the balance weight.

MMDTL2 reduces to MMDT when $c_f = 1$ and $c_d = 0$.

\end{problem}

We add weight $c_f$ to the first term of Eq. (\ref{eq:w_subproblem_mmdtl2}) because $c_t$ is used for both Eqs. (\ref{eq:w_subproblem_mmdtl2}) and (\ref{eq:svm_subproblem_mmdtl2}). In both equations, $c_t$ modifies the balance between the first and second terms. Therefore, without $c_f$, the first terms of Eqs. (\ref{eq:w_subproblem_mmdtl2}) and (\ref{eq:svm_subproblem_mmdtl2}) have the same weight, which might not be suitable in general.
We will investigate the effect of $c_f$ in section 4.4.

The SVM subproblem (\ref{eq:svm_subproblem_mmdtl2}) can be solved by common SVM solvers, as is done for (\ref{eq:svm_subproblem_mmdt}) by MMDT.
Therefore, in the following sections, we focus on deriving the primal and dual forms of subproblem (\ref{eq:w_subproblem_mmdtl2}) as standard QP problems.

\section{Proposed method}

In this section, we highlight the main results only.
Full derivations and proofs are given in the appendix.

\subsection{Primal problem}

First, we rephrase subproblem (\ref{eq:w_subproblem_mmdtl2})
with inequality constraints instead of loss functions.

\begin{problem}[Estimation of $W$]

We want to find $W \in \mathbb{R}^{L_s \times (L_t+1)}$ that minimizes
the following objective function:
\begin{align}
  \min_{W, \{ \xi_{km}^t \} } &
  \frac{1}{2} c_f \| W \|_F^2
  +
  c_T \sum_{k=1}^K \sum_{m=1}^M \xi_{km}^t \notag\\
  & + \frac{1}{2}   \sum_{m=1}^M \sum_{n=1}^N s_{nm} \| W \hat\x_m^t - \x_n^s \|_2^2
  \\
  \text{s.t.} & 
  \quad
  \xi_{km}^t \ge 0,
  \quad
  y_{km}^t \hat\th_k^T 
  \begin{pmatrix}
  W \hat\x_m^t \\
  1
  \end{pmatrix}
   -1 + \xi_{km}^t \ge 0.
\end{align}

\end{problem}

\subsection{Primal QP problem}

The problem can be written in the form of a canonical QP problem.

\begin{lemma}

The problem can be written as
\begin{align}
  & \min_{\w, \bxi }
  \frac{1}{2} (\w^T, \bxi^T) 
  \begin{pmatrix}
  V & 0 \\
  0 & 0
  \end{pmatrix}
  \begin{pmatrix}
  \w \\ \bxi
  \end{pmatrix}
  +  (-\q^T, c_T \bm{1}_{KM}^T) \begin{pmatrix}
  \w \\ \bxi
  \end{pmatrix} + \frac{1}{2}  s
    \\
  & \text{s.t.} 
  \quad
  \begin{pmatrix}
  0 & I_{KM} \\
  Y \Phi^T & I_{KM}
  \end{pmatrix}
  \begin{pmatrix}
    \w \\ \bxi
  \end{pmatrix} 
   \ge  
  \begin{pmatrix}
    \bm{0} \\
    \bm{1}_{KM} -  Y\bb    
  \end{pmatrix},
\end{align}
where variables are defined in the proof.
\end{lemma}

This primal QP problem involves very large matrices that are impractical to compute.
More precisely, $V$ is a matrix of size $L_s (L_t+1) \times L_s (L_t + 1)$, which can be very large when the dimensions of features ($L_s$ and $L_t$) are large.
Next, we therefore derive the dual form of the problem, which we expect to be less expensive to compute.

\subsection{Dual problem}

The dual of the problem is given in the corollary below.

\begin{corollary}
The dual form of the original primal problem is given by
\begin{align}
\max_{\a} &
 \frac{-1}{2}  \bm{a}^T  Y^T ( (\Theta^T \Theta) \otimes G ) Y \bm{a} \notag \\
   & + (\bm{1}^T - (Y\bb )^T - \vec \left(\Theta^T X^s S G \right)^T
   Y) \bm{a}
 \\
& \text{s.t.} \quad c_T \bm{1} \ge \a \ge 0.
\end{align}
\end{corollary}

\subsection{Retrieving the primal solution}

After solving the dual problem with a QP solver,
we need to convert the dual solution $\a$ to $\w$ and $\b$ by
\begin{align}
  \w &= V^{-1} (\q + \Phi Y \a),
\end{align}
then finally to $W$.

Here, we again face the problem of large matrix $V$,
that should not be used.
In the corollary below, we show primal solution $W$ directly,
i.e., avoiding conversions from $\a$ to $\w$, and then to $W$.

\begin{corollary}
The solution to the primal problem is given by
\begin{align}
  W &=
  \left( X^s S
  +
  \Theta ( \Upsilon \odot \Lambda )^T  \right)
  (X^t)^T A^{-1},
\end{align}
where $\odot$ is element-wise multiplication
and variables are given in the proof.
\end{corollary}

\subsection{Kernelization}

We derive the kernel version of the dual formulation.
To do so,
we apply $W$ to target sample $\x^t$ by multiplying it from the left, i.e., $W \hat{\x}^t$;
Therefore, all computations with target samples are inner products,
which means we can use kernels to replace the inner products.

To transform target sample $\x^t$ with $W$, we have
\begin{align}
  W \hat{\x}^t
  =&
  \left( X^s S
  +
  \Theta ( \Upsilon \odot \Lambda )^T  \right) \notag \\ &
  \left(\frac{1}{c_f} I_{M} - \frac{1}{c_f^2} K^t (S_M^{-1} + \frac{1}{c_f} K^t)^{-1} \right) 
  \begin{pmatrix}
  k(\x^t_1, \x^t) \\
  \vdots\\
  k(\x^t_M, \x^t)
  \end{pmatrix}.
\end{align}

\section{Results and discussions}

In this section, we show experimental results
by using datasets consisting of two NBI endoscopic devices
to understand how our proposed method behaves given
differing numbers of target domain training samples.

We used the following parameter values throughout the experiments unless stated otherwise:
$c_f = c_t = c_s = c_d = 0.1$ for MMDTL2 based on our preliminary experiments,
and $c_s = 0.05$ and $c_t = 1$ for MMDT, as used in the code of the original MMDT
\cite{Hoffman2013a,Hoffman2013d,Hoffman2014a}.

\def\ppm{$\!\!\pm$}
\def\aa{\textsf{\tiny **}}
\def\aaa{\textsf{\tiny *}}
\def\cc{\textsf{\tiny ++}}
\def\ccc{\textsf{\tiny +}}

\subsection{NBI endoscopic image dataset}


The NBI dataset 
used in this experiment consisted of two domains. For the first (i.e., source) domain, we used the NBI image dataset consisting
of 908 NBI patches collected from endoscopic examinations at Hiroshima University
by using OLYMPUS EVIS LUCERA endoscope system \cite{CV-260};
patches were labeled based on NBI magnification findings \cite{Kanao2009,Oba2010},
which categorizes appearances of tumors into types A, B, and C, with type C further sub-classified into C1, C2, and C3 based on microvessel structures (see Figure \ref{fig:nbi_magnification}). In this study, we used only types A, B, and C3 in accordance with our previous work \cite{Tamaki2013,Hirakawa2014,Sonoyama2015a,Sonoyama2015b}.
In general, a patch is trimmed from a larger frame of the entire endoscopic image such that the trimmed rectangular region represents the typical texture pattern of the colorectal polyp appearing in the frame.
To align the size of patches between source and target domains, we further trimmed the center of each patch to $180 \times 180$ and
created 734 patches with 289 in type A, 365 in type B, and 80 in type C3.
Note that this study was conducted with the approval from the Hiroshima University Hospital ethics committee, and informed consents were obtained from the patients and/or family members for the endoscopic examinations.

For the second (i.e., target) domain, we used another NBI image
dataset consisting of 279 patches with 92 in type A and 187 in type B.
These images were taken by OLYMPUS EVIS LUCERA ELITE endoscope system \cite{CV-290}, which is a newer model than that of the system of the source domain.
Due to the limited number of endoscopic examinations using this newer endoscope system, we trimmed the center square to $180 \times 180$
from video frames of 41 NBI examination videos;
hence, there are two factors of domain shift here: (1) the NBI endoscopic devices and (2) the differences between still images (i.e., source) and video frames (i.e., target).
Note that type C3 samples obtained using the new endoscopy are not yet available because the number of type C3 samples, which corresponds to a developed cancer \cite{Tamaki2013}, is smaller in general, and the new endoscopy of our facility has not yet had the opportunity to capture an image of a developed cancer of this type.

From these two domains, we computed convolutional neural network (CNN) features extracted using CaffeNet \cite{jia2014caffe}; more specifically, this is the fc6 feature of 4096 dimensions, which is known to work well for many tasks \cite{Razavian_2014_CVPR_Workshops}. These features were used without dimensionality reduction.

To see the effect of the number of target samples,
we prepared training/test sets as follows.
First, we randomly split 
the source and target domain samples into half;
we therefore had a source training set of 367 source samples
and a target training set of 139 target samples.
Next, we kept a specified number of target 
samples per category (up to 40) in each of the target training sets,
discarding the rest.
For the test set, we used 140 target samples.
We created 10 training/test sets and reported average performance.

\begin{figure}[t]
\includegraphics[width=\linewidth]{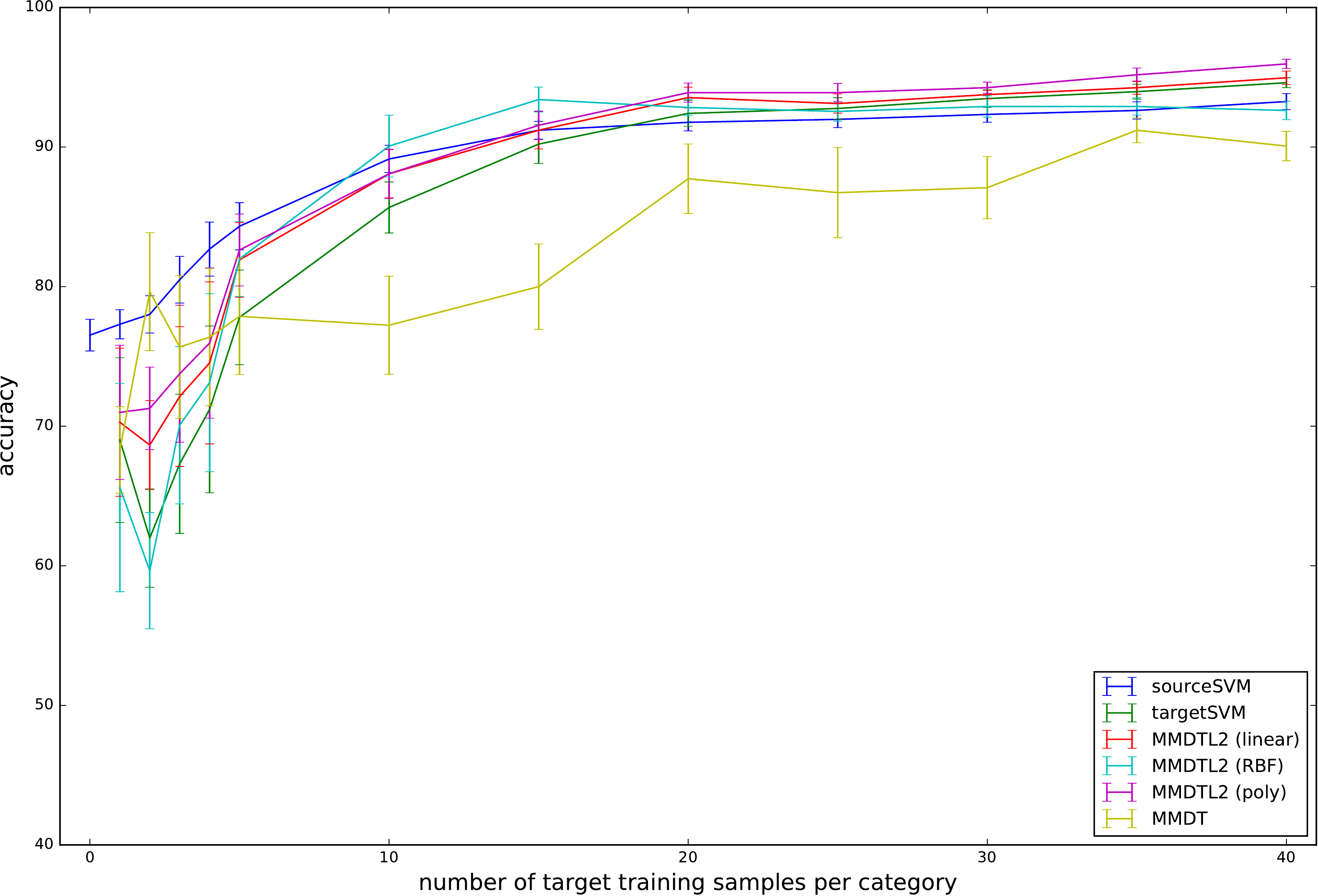}
\caption{Experimental results for the NBI datasets with the fc6 feature.}
\label{fig:NBI_result}
\end{figure}

Figure \ref{fig:NBI_result} shows performance results of different methods
over the different numbers of target training samples.
As a baseline, we also show two SVM results, i.e., source SVM and target SVM.
Source SVM estimates an SVM classifier with source and target training samples (and only source samples when no target samples are used).
Target SVM does the same, but only with target training samples.
MMDT results were obtained using the code provided by \cite{Hoffman2013d}, just as in our first experiment.
There were three results for MMDTL2, i.e., a linear version and two kernel versions with RBF and polynomial kernels.

Even with no target samples, source SVM performed relatively better and increased its performance as the number of target training samples increased. Target SVM started below 70\%, but caught up to source SVM when 20 training target samples were given.
The results indicate that our proposed MMDTL2 behaves between source and target SVMs. When one or two target samples are given, MMDTL2 behaves similarly to target SVM and far below the source SVM, but it quickly gains from the benefits of adaptation. With the RBF kernel, MMDTL2 is the best when 10 and 15 training samples are given, but the linear and polynomial kernels become better when more samples are given. MMDTL2 with the linear kernel and target SVM approach one another, which is expected because a sufficient number of target training samples are considered to be the best for classifying target domain samples.
Overall, MMDTL2 with a polynomial kernel works the best.

\begin{table*}[t]
\centering
\caption{Experimental results for the NBI datasets with the fc6 feature.}
\label{tab:NBI_result}
\scriptsize
\begin{tabular}{c|llllll}
 & & & & MMDTL2 & MMDTL2 & MMDTL2 \\
 & sourceSVM & targetSVM & MMDT & (linear) & (RBF) & (poly) \\ \hline
1   & 77.31 \ppm 1.05 \cc& 69.01 \ppm 5.90 & 68.30 \ppm 3.10  & 70.28 \ppm 5.31 & 65.60 \ppm 7.46  & 70.99 \ppm 4.81 \\
2   & 78.01 \ppm 1.34 \cc& 61.99 \ppm 3.52 & 79.65 \ppm 4.22  & 68.65 \ppm 3.18 \cc  & 59.65 \ppm 4.16 & 71.28 \ppm 2.96 \cc\\
3   & 80.50 \ppm 1.67 \cc& 67.30 \ppm 4.98 & 75.67 \ppm 5.11 \cc & 72.13 \ppm 5.01  & 70.07 \ppm 5.63 & 73.76 \ppm 4.91 \ccc\\
4   & 82.69 \ppm 1.94 \cc& 71.21 \ppm 5.98 & 76.38 \ppm 4.92 \cc & 74.54 \ppm 5.80  & 73.12 \ppm 6.38 & 75.96 \ppm 5.38 \\
5   & 84.33 \ppm 1.69 \cc& 77.80 \ppm 3.39 & 77.87 \ppm 4.17    & 81.91 \ppm 2.68 \ccc & 81.99 \ppm 2.69 \ccc  & 82.62 \ppm 2.58 \cc\\
10  & 89.15 \ppm 0.97 \cc& 85.67 \ppm 1.83 & 77.23 \ppm 3.52    & 88.09 \ppm 1.73 \ccc  & 90.07 \ppm 2.22 \cc  & 88.09 \ppm 1.76 \ccc\\
15  & 91.21 \ppm 0.64 & 90.21 \ppm 1.39 & 80.00 \ppm 3.06    & 91.21 \ppm 1.33  & 93.40 \ppm 0.89 \aa \cc & 91.56 \ppm 1.02 \\
20  & 91.77 \ppm 0.61 & 92.41 \ppm 0.93 & 87.73 \ppm 2.49    & 93.55 \ppm 0.74 \aa \ccc   & 92.84 \ppm 0.58 \aa& 93.90 \ppm 0.69 \aa \cc\\
25  & 91.99 \ppm 0.59 & 92.77 \ppm 0.78 \aaa  & 86.74 \ppm 3.23 & 93.12 \ppm 0.69 \aa   & 92.55 \ppm 0.70  & 93.90 \ppm 0.66 \aa \cc\\
30  & 92.34 \ppm 0.56 & 93.48 \ppm 0.63 \aa & 87.09 \ppm 2.23  & 93.76 \ppm 0.30 \aa  & 92.91 \ppm 0.76 & 94.26 \ppm 0.40 \aa \cc\\
35  & 92.62 \ppm 0.62 & 93.97 \ppm 0.53 \aa & 91.21 \ppm 0.89  & 94.26 \ppm 0.48 \aa  & 92.91 \ppm 0.63 & 95.18 \ppm 0.48 \aa \cc\\
40  & 93.26 \ppm 0.57 & 94.61 \ppm 0.35 \aa & 90.07 \ppm 1.05  & 94.96 \ppm 0.48 \aa  & 92.62 \ppm 0.65 & 95.96 \ppm 0.34 \aa \cc\\
\end{tabular}
\end{table*}

To make the discussion above more quantitative, we performed the one-tailed Welch's unequal variances $t$-test \cite{Welch1951,Kanji2006} to show the statistical significance of the proposed method with respect to sourceSVM and targetSVM. Table \ref{tab:NBI_result} shows the same performance results presented in Figure \ref{fig:NBI_result}, but with results of the statistical test. Tests with 1\% and 5\% significance levels (i.e., $p < 0.01$ and $p < 0.05$) are indicated by $**$ and $*$, respectively, with respect to sourceSVM. Similarly, $++$ and $+$ indicate the significance of test results with respect to targetSVM. MMDTL2 with a polynomial kernel is better than sourceSVM and targetSVM when $M \ge 20 $ ($M$ is is the number of target samples given per category). MMDTL2 with the RBF kernel works well only when $20 \ge M \ge 5$. MMDTL2 is better than sourceSVM when $M \ge 20$.


\subsection{NBI endoscopic image dataset with high-dimension features}

\begin{figure}[t]
\includegraphics[width=\linewidth]{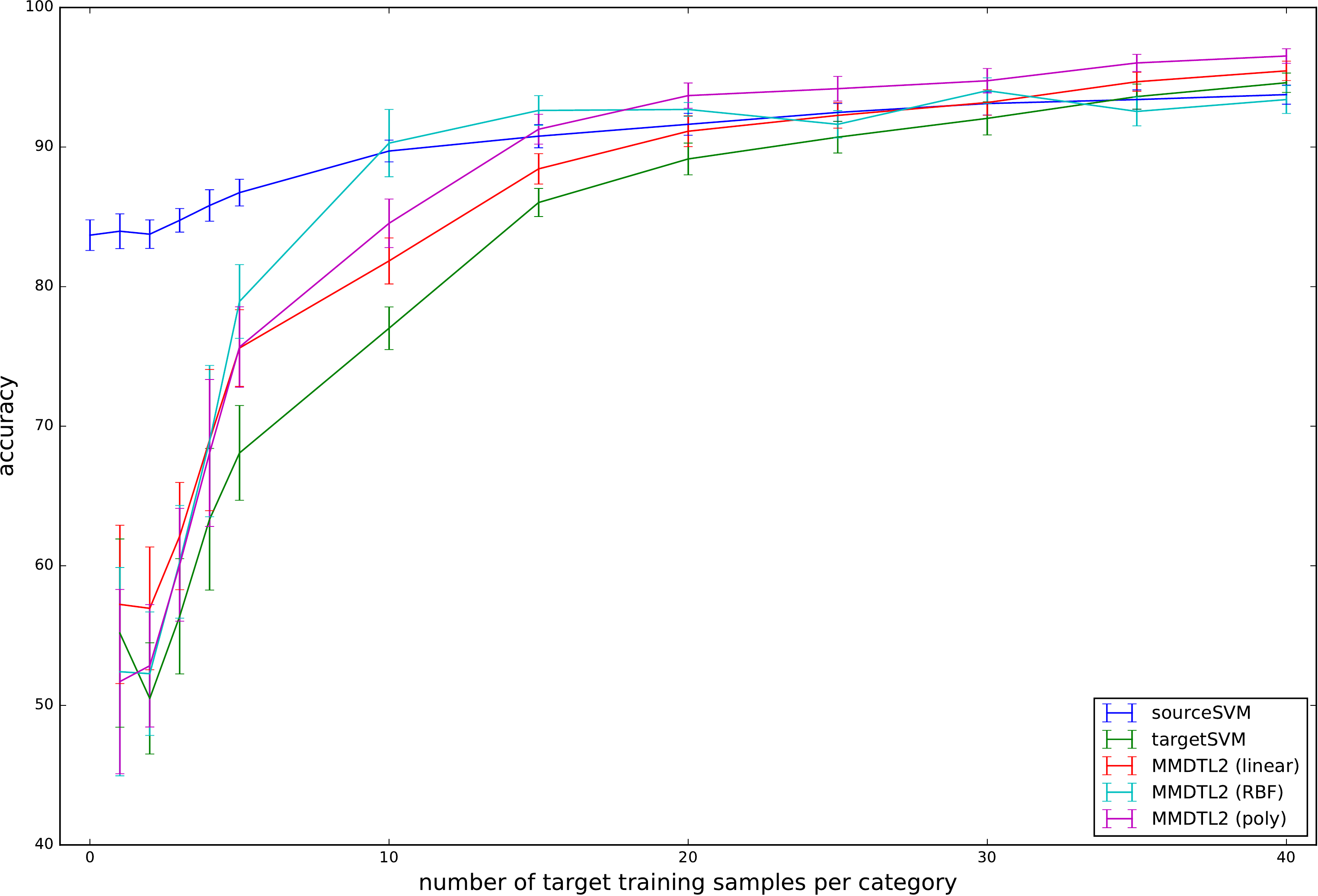}
\caption{Experimental results for the NBI datasets with the conv3 feature.}
\label{fig:NBI_result_conv3}
\end{figure}

Figure \ref{fig:NBI_result_conv3} shows performance results of the given methods
as in Figure \ref{fig:NBI_result} with the same protocols for training and evaluation.
The difference here is the set of features used; more specifically, we use the conv3 features of 64,896 dimensions rather than fc6.
For NBI patch classification problems \cite{Tamaki2016a},
we have seen that conv3 features worked better than fc6 features.
Therefore we choose conv3 features for this experiment.
However, transformation matrix $W$ for the conv3 features could be very large without our efficient dual formulation.
The ability to handle such large dimensions of features is the key advantage of our proposed dual MMDTL2.
In contrast, it is not possible to train MMDT because it involves a matrix $W$ of size $64,896 \times 64,897$ for storage (approximately 31 GB for double-type elements) and requires more space for working memory. The test phase is also impractical because MMDT uses $W$ expertly for converting target samples.
Figure \ref{fig:NBI_result_conv3} shows that the sourceSVM is the best when a few target samples are available, just as in Figure \ref{fig:NBI_result}. Further, our proposed MMDTL2 with a polynomial kernel becomes better at the right half of the plot, and the differences between sourceSVM and targetSVM are much more significant than those in Figure \ref{fig:NBI_result}.

Table \ref{tab:NBI_result_conv3} shows the same performance results as shown in Figure \ref{fig:NBI_result_conv3} along with the results of the one-tailed Welch's $t$-test \cite{Welch1951,Kanji2006}. The meaning of the marks is the same as that in Table \ref{tab:NBI_result}. Again, MMDTL2 with a polynomial kernel is better than sourceSVM and targetSVM when $M \ge 20 $. MDTL2 with the linear kernel becomes better than sourceSVM and targetSVM when $M \ge 35$.

\begin{table*}[t]
\centering
\caption{Experimental results for the NBI datasets with the conv3 feature.}
\label{tab:NBI_result_conv3}
\scriptsize
\begin{tabular}{c|lllll}
 & & & MMDTL2 & MMDTL2 & MMDTL2 \\ 
 & sourceSVM & targetSVM & (linear) & (RBF) & (poly) \\ \hline
1 & 83.97 \ppm 1.25 \cc& 55.18 \ppm 6.74 & 57.23 \ppm 5.67 & 52.41 \ppm 7.46 & 51.70 \ppm 6.60 \\
2 & 83.76 \ppm 1.02 \cc& 50.50 \ppm 3.98 & 56.95 \ppm 4.39 \cc& 52.27 \ppm 4.43 & 52.84 \ppm 4.39 \\
3 & 84.75 \ppm 0.84 \cc& 56.38 \ppm 4.13 & 62.13 \ppm 3.84 \cc& 60.28 \ppm 4.04 & 60.07 \ppm 4.04 \\
4 & 85.82 \ppm 1.12 \cc& 63.33 \ppm 5.08 & 69.01 \ppm 5.06 \ccc& 68.94 \ppm 5.42 & 68.09 \ppm 5.26 \\
5 & 86.74 \ppm 0.96 \cc& 68.09 \ppm 3.39 & 75.60 \ppm 2.75 \cc& 78.94 \ppm 2.64 \cc& 75.67 \ppm 2.89 \cc\\
10 & 89.72 \ppm 0.78 \cc& 77.02 \ppm 1.53 & 81.84 \ppm 1.65 \cc& 90.28 \ppm 2.41 \cc& 84.54 \ppm 1.74 \cc\\
15 & 90.78 \ppm 0.83 \cc& 86.03 \ppm 1.01 & 88.44 \ppm 1.09 \cc& 92.62 \ppm 1.06 \aa \cc& 91.28 \ppm 1.07 \cc\\
20 & 91.63 \ppm 0.78 \cc& 89.15 \ppm 1.14 & 91.13 \ppm 1.10 \cc& 92.70 \ppm 0.50 \aa \cc& 93.69 \ppm 0.91 \aa \cc\\
25 & 92.48 \ppm 0.64 \cc& 90.71 \ppm 1.14 & 92.27 \ppm 0.91 \cc& 91.63 \ppm 0.98 & 94.18 \ppm 0.88 \aa \cc\\
30 & 93.12 \ppm 0.83 & 92.06 \ppm 1.19 & 93.19 \ppm 0.91 & 94.04 \ppm 0.92 \cc& 94.75 \ppm 0.88 \aa \cc\\
35 & 93.40 \ppm 0.69 & 93.62 \ppm 0.90 & 94.68 \ppm 0.68 \aa \ccc& 92.55 \ppm 1.03 & 96.03 \ppm 0.62 \aa \cc\\
40 & 93.76 \ppm 0.68 & 94.61 \ppm 0.69 & 95.46 \ppm 0.69 \aa \ccc& 93.40 \ppm 1.00 & 96.52 \ppm 0.52 \aa \cc\\
\end{tabular}
\end{table*}

\subsection{NBI endoscopic image dataset with features of different dimensions}

Figure \ref{fig:NBI_result_diff_dim} shows another set of performance results. Here, we use different feature dimensions for source and target samples. In particular, we use the conv3 features of 64,896 dimensions for the source samples and the conv5 features of 43,264 dimensions for the target samples. In this case, MMDT does not work because of its high memory cost as described in the previous section, and furthermore, sourceSVM cannot be used because of the difference of feature dimensions. Table \ref{tab:NBI_result_diff_dim} shows the same performance results along with the results of the one-tailed Welch's $t$-test \cite{Welch1951,Kanji2006} using the same marks as those in Table \ref{tab:NBI_result}. MMDTL2 with a polynomial kernel is better than targetSVM again, when $n \ge 10 $. MMDTL2 with the linear and RBF kernels becomes better than targetSVM when $25 \ge n \ge 10$.

\begin{figure}[t]
\includegraphics[width=\linewidth]{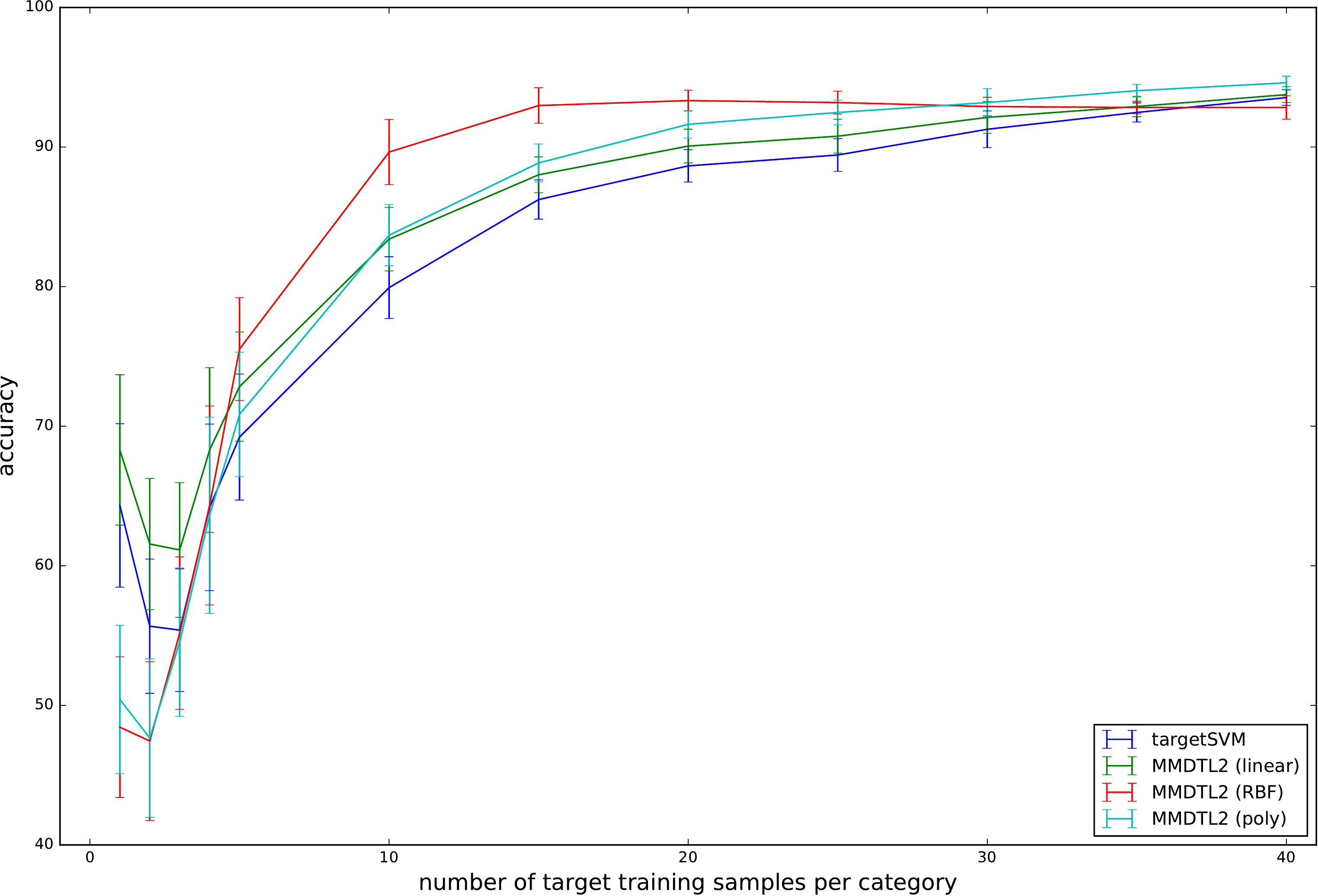}
\caption{Experimental results for the NBI datasets with the different features in source and target domains.}
\label{fig:NBI_result_diff_dim}
\end{figure}

\begin{table*}[t]
\centering
\caption{Experimental results for the NBI datasets with the different features in source and target domains.}
\label{tab:NBI_result_diff_dim}
\scriptsize
\begin{tabular}{c|llll}
 & targetSVM & MMDTL2 & MMDTL2 & MMDTL2 \\
 &           & (linear)& (RBF) & (poly) \\ \hline
1 & 64.33 \ppm 5.86 & 68.30 \ppm 5.39 & 48.44 \ppm 5.04 & 50.43 \ppm 5.32 \\
2 & 55.67 \ppm 4.81 & 61.56 \ppm 4.70 \ccc & 47.45 \ppm 5.69 & 47.66 \ppm 5.68 \\
3 & 55.39 \ppm 4.40 & 61.13 \ppm 4.83 \ccc & 55.18 \ppm 5.46 & 54.54 \ppm 5.32 \\
4 & 64.18 \ppm 5.97 & 68.30 \ppm 5.90 & 64.33 \ppm 7.13 & 63.62 \ppm 7.03 \\
5 & 69.22 \ppm 4.51 & 72.84 \ppm 3.91 & 75.53 \ppm 3.68 \cc & 70.85 \ppm 4.46 \\
10 & 79.93 \ppm 2.21 & 83.40 \ppm 2.28 \cc & 89.65 \ppm 2.34 \cc & 83.69 \ppm 2.19 \cc \\
15 & 86.24 \ppm 1.41 & 88.01 \ppm 1.28 \ccc & 92.98 \ppm 1.28 \cc & 88.87 \ppm 1.35 \cc \\
20 & 88.65 \ppm 1.16 & 90.07 \ppm 1.21 \ccc & 93.33 \ppm 0.75 \cc & 91.63 \ppm 0.98 \cc \\
25 & 89.43 \ppm 1.17 & 90.78 \ppm 1.21 \ccc & 93.19 \ppm 0.82 \cc & 92.48 \ppm 0.89 \cc \\
30 & 91.28 \ppm 1.32 & 92.13 \ppm 1.14 & 92.91 \ppm 0.66 \cc & 93.19 \ppm 0.99 \cc \\
35 & 92.48 \ppm 0.69 & 92.91 \ppm 0.73 & 92.84 \ppm 0.45 & 94.04 \ppm 0.45 \cc \\
40 & 93.55 \ppm 0.55 & 93.76 \ppm 0.57 & 92.84 \ppm 0.84 & 94.61 \ppm 0.48 \cc \\
\end{tabular}
\end{table*}

\subsection{Effects of parameter values}

We introduced parameters $c_f$ and $c_d$ for our MMDTL2 formulation in Section 2. Figure \ref{fig:NBI_result_cf} shows performance results for different values of $c_f$, whereas other parameters are fixed to
their default values. The results are almost identical when $c_f > 10^{-7}$; however, the performance becomes unsatisfactory for smaller values of $c_f$. This is because the inverse of $c_f$ is used in the formulation of MMDTL2 (see the Appendix). This suggests that a small amount of regularization for $\| W \|_F$ might help improve performance.
Note that the instability of the results when $c_f \le 10^{-7}$
may be due to the fact that variables of single precision have 8 digits of precision,
and the inverses of smaller values lead to numerical instability.
Therefore, we can say that results are insensitive to the choice of $c_f$ as long as larger values are used.

Figure \ref{fig:NBI_result_cd} shows the performance results for different values of $c_d$, whereas other parameters are fixed to their default values. To see the differences between the different values, the results for 20 target samples are shown in Figure \ref{fig:NBI_result_cd_n20}. The results are almost identical when $c_d \ge 10^{-6}$; however, the performance decreases for $10^{-7} \ge c_d > 10^{-10}$. 
A detailed investigation of this effect is one of our future tasks; 
however, it is clear that the performance is better and stable for larger values of $c_d$.

\begin{figure}[t]
\includegraphics[width=\linewidth]{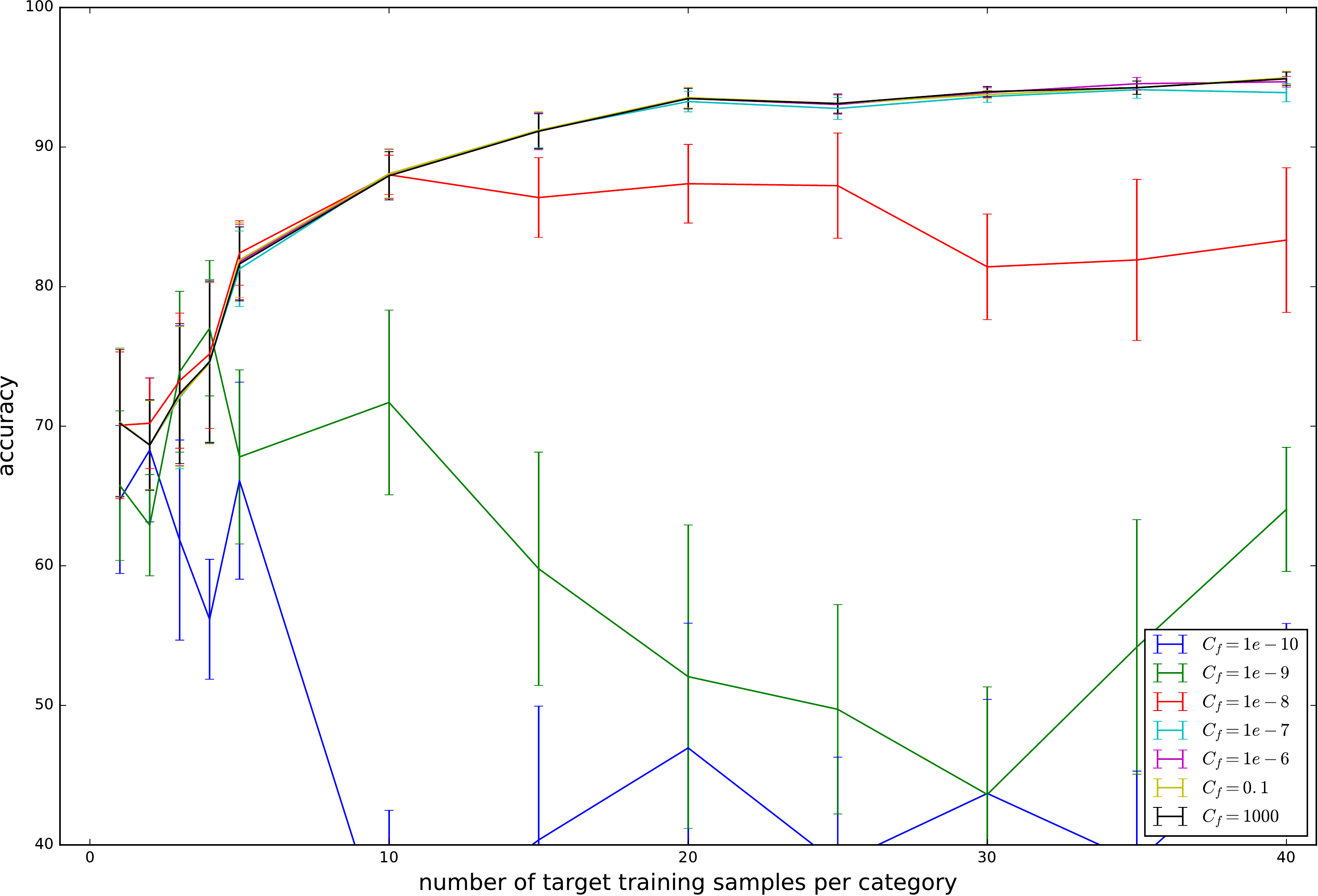}
\caption{Effect of $c_f$ for the NBI datasets with the fc6 feature.}
\label{fig:NBI_result_cf}
\end{figure}

\begin{figure}[t]
\includegraphics[width=\linewidth]{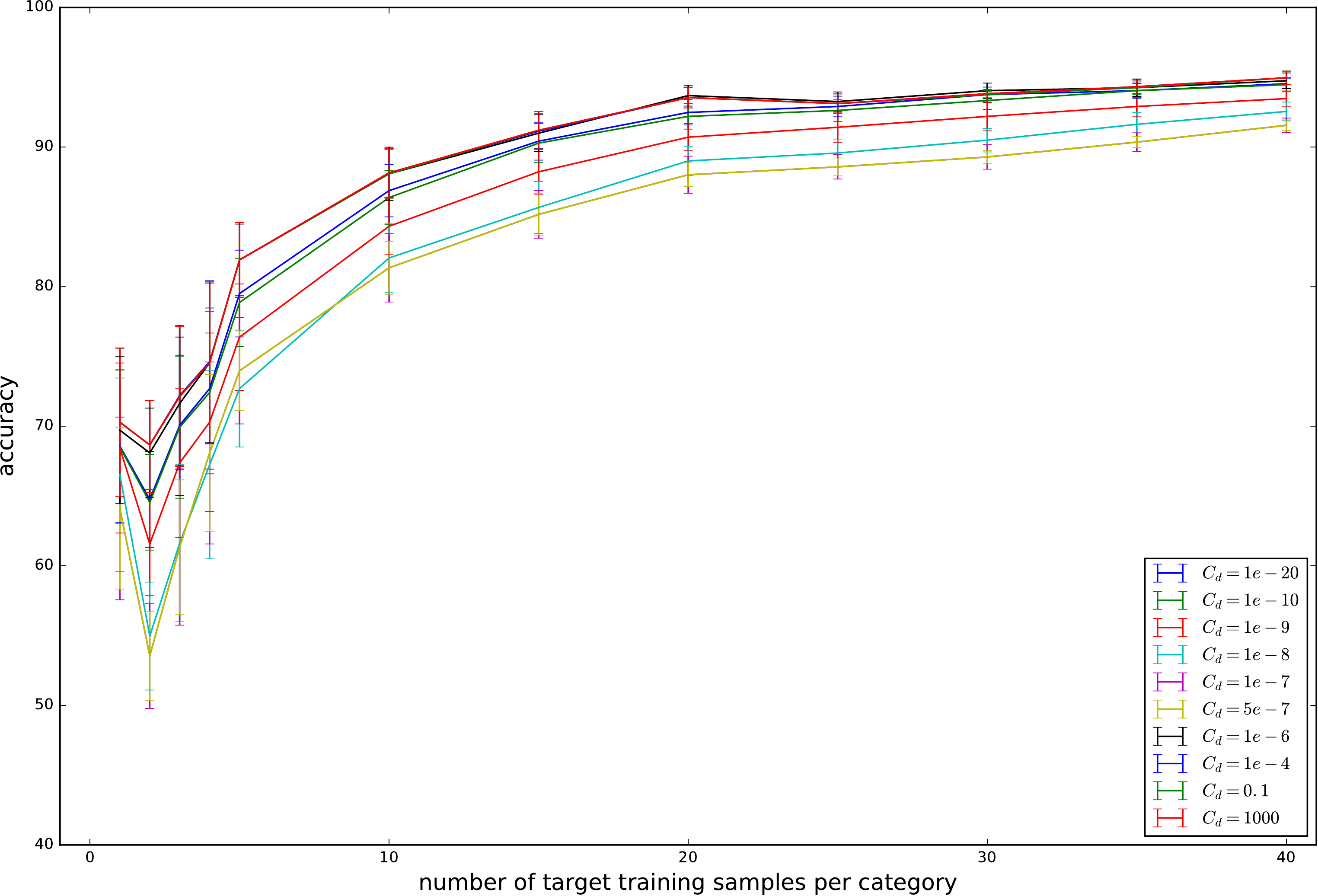}
\caption{Effect of $c_d$ for the NBI datasets with the fc6 feature.}
\label{fig:NBI_result_cd}
\end{figure}

\begin{figure}[t]
\includegraphics[width=\linewidth]{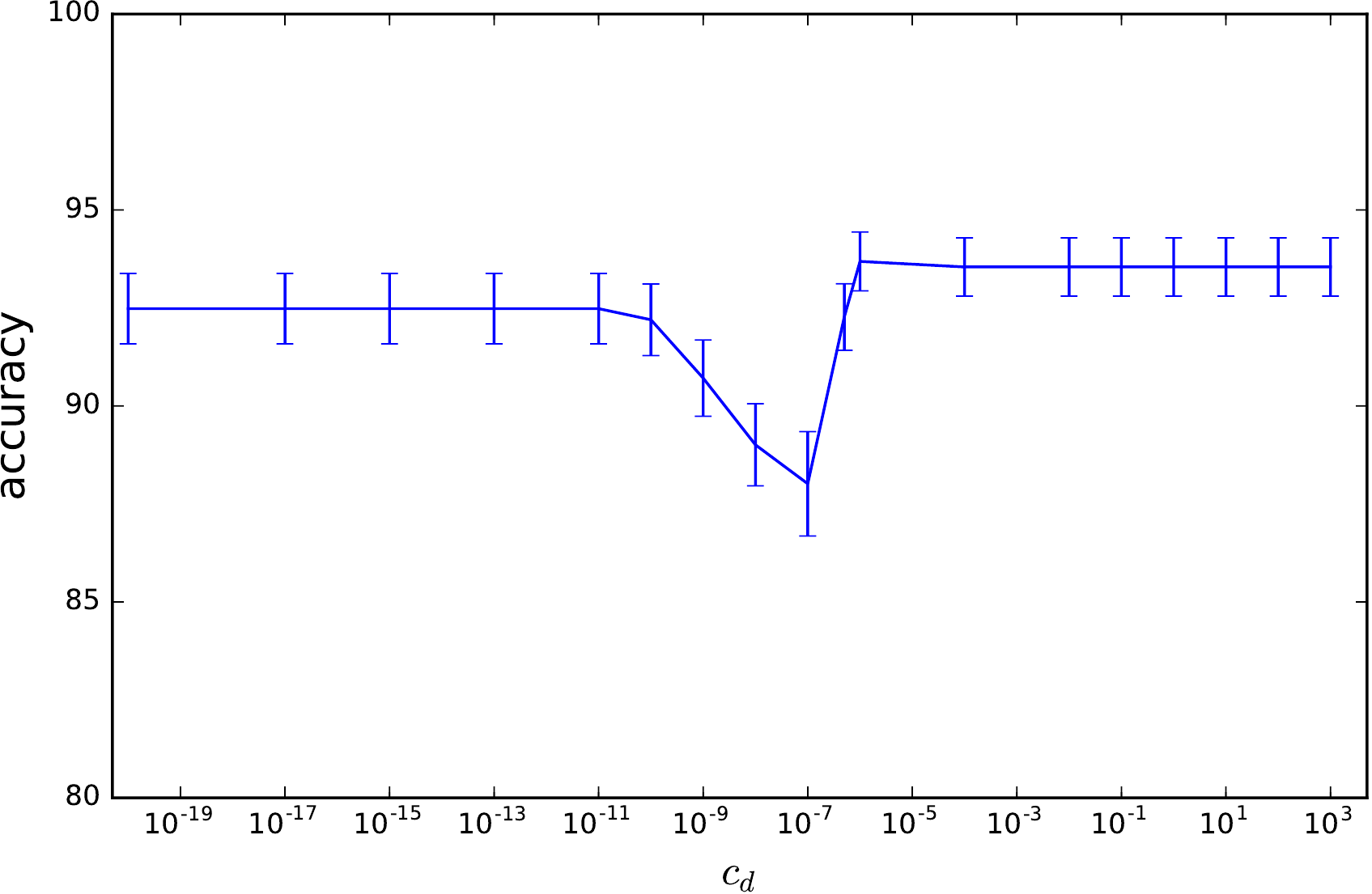}
\caption{Effect of $c_d$ for the NBI datasets with the fc6 feature and 20 training samples.}
\label{fig:NBI_result_cd_n20}
\end{figure}

\subsection{Computation time}

Figure \ref{fig:NBI_result_time} shows the computation time for the training phases of each method over different feature dimensions ($L_s = L_t$) using MATLAB implementations on an Intel Core i7 3.4 GHz CPU with 16 GB memory. The primal MMDTL2 involves quite a large matrix for $V$, and it is not practical when $L_s > 180$ in terms of memory and computation time. The cost of the dual MMDTL2 depends on the number of target samples; however, it is quite fast compared to the primal MMDTL2 for both 10 and 40 target training samples. The dual MMDTL2 can deal with much higher feature dimensions, while MMDT exceeds the memory limitation when $L_s > 10000$ because matrix $W$ must be explicitly stored.

\begin{figure}[t]
\includegraphics[width=\linewidth]{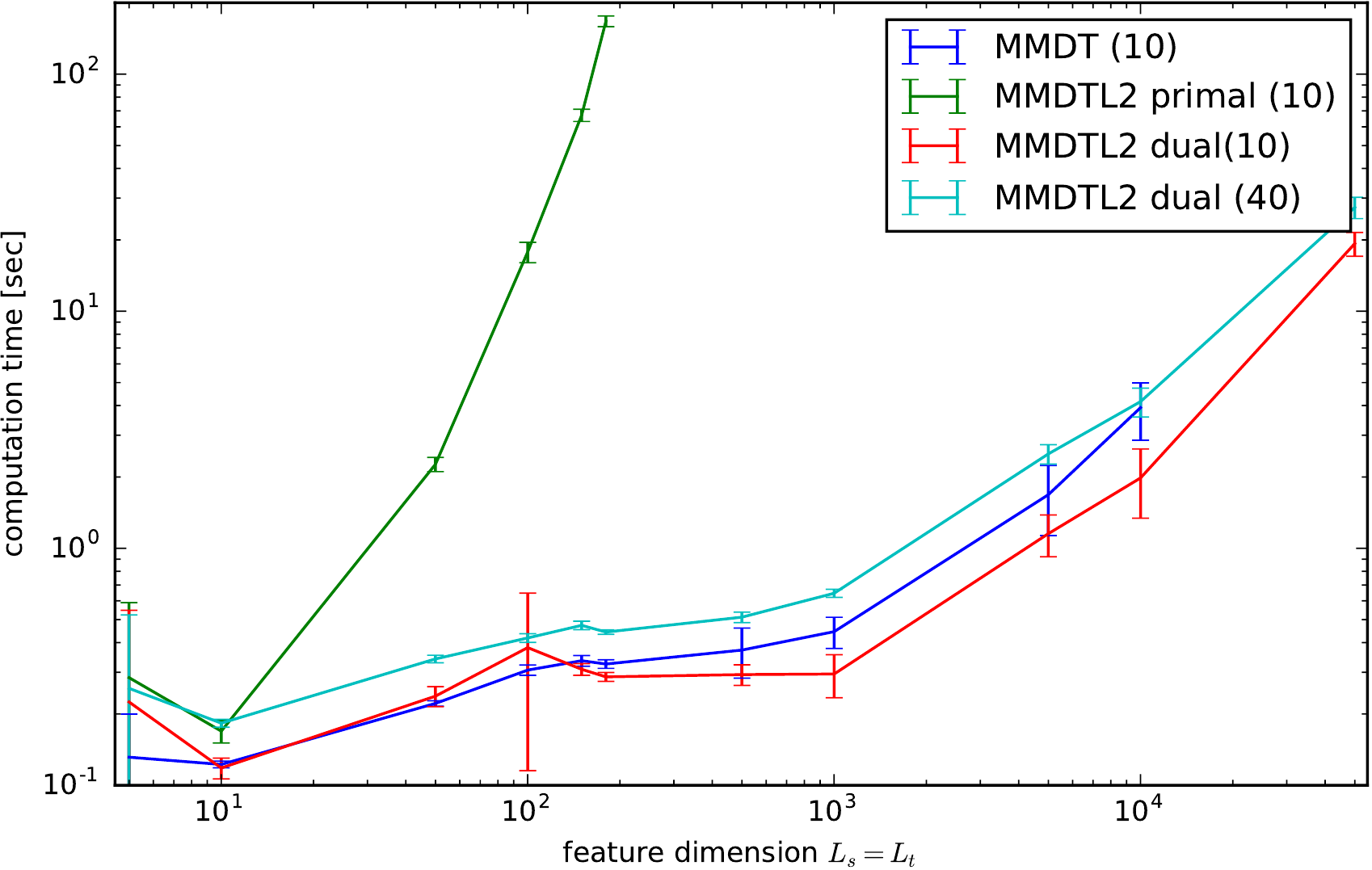}
\caption{Computation time. The numbers of target samples used are shown in braces.}
\label{fig:NBI_result_time}
\end{figure}

\section{Conclusions}

In this paper, we proposed MMDT with $L_2$ constraints, i.e., MMDTL2, deriving the dual formulation
with much lesser computational costs as compared to the naive QP problem.
Further, we showed the kernelization of our method.
Experimental results with NBI datasets from two different endoscopic devices showed that our proposed MMDTL2 with linear and polynomial kernels
performed better than the given baselines (i.e., source and target SVMs).
Our future work includes using other loss functions for problem formulation. We observed that the one-vs-rest multiclass classification by SVMs was a performance bottleneck of MMDTL2 in our experiments. 
Therefore, instead of relying on maximum margin loss functions, multiclass logistic loss might be better here. In the future, we plan to explore this idea and report performance results for the NBI dataset as well.
In addition, we plan to investigate parameter tuning with cross validation and compare the proposed method with other adaptation methods such as that in \cite{Tzeng2015}.


%

\appendices


\setcounter{problem}{3}
\setcounter{corollary}{0}
\setcounter{lemma}{0}

\section{Primal problem}

In this section, we rephrase subproblem (\ref{eq:w_subproblem_mmdtl2})
with inequality constraints instead of loss functions.

\begin{problem}[Estimation of $W$]
\label{prob:estimation_w}

We want to find $W \in \mathbb{R}^{L_s \times (L_t+1)}$ that minimizes
the following objective function:
\begin{align}
  \min_{W, \{ \xi_{km}^t \} } &
  \frac{1}{2} c_f \| W \|_F^2
  +
  c_T \sum_{k=1}^K \sum_{m=1}^M \xi_{km}^t \notag \\
  & + \frac{1}{2}   \sum_{m=1}^M \sum_{n=1}^N s_{nm} \| W \hat\x_m^t - \x_n^s \|_2^2
  \label{eq:estimation_w_cost}
  \\
  & \text{s.t.} \notag
  \\
  & \xi_{km}^t \ge 0,
  \label{eq:estimation_w_condition1}
  \\
  & y_{km}^t \hat\th_k^T 
  \label{eq:estimation_w_condition2}
  \begin{pmatrix}
  W \hat\x_m^t \\
  1
  \end{pmatrix}
   -1 + \xi_{km}^t \ge 0.
\end{align}

\end{problem}

First, we rewrite the objective function in a matrix form.
To this end, we introduce the vec operator and some formulas below.

\subsection{Operator vec}

Here, we define a vectorized operator for rearranging matrix-vector products.

\begin{define}
\label{def:vec}
For a given matrix $W \in \mathbb{R}^{L_s \times (L_t+1)}$,
denoted by a set of row vectors $\w_i \in \mathbb{R}^{L_t}$ as 
\begin{align}
  W =
  \begin{pmatrix}
    \w_1^T \\
    \w_2^T \\
    \vdots \\
    \w_{L_s}^T
  \end{pmatrix},
\end{align}
we define operator $\vec$, which vectorizes $W$ in the row-major order as
\begin{align}
  \vec(W) =
  \begin{pmatrix}
      \w_1 \\
      \w_2 \\
      \vdots \\
      \w_{L_s}
    \end{pmatrix}
    \in \mathbb{R}^{L_s (L_t+1)}.
\end{align}
\end{define}

This definition is different from the one used in the literature, which is defined in the column-major order, for example, in \cite{Harville1997}.

Next, we can rewrite matrix-vector multiplications using the vec operator, as
summarized in the following lemma.

\begin{lemma}
\label{lemma:multiplication_with_vec}
For given matrix $W \in \mathbb{R}^{L_s \times (L_t+1)}$
and vectors $\x \in \mathbb{R}^{L_t +1}$ and $\z \in \mathbb{R}^{L_s}$,
the following equations hold:
\begin{align}
  W \hat\x &= (I_{L_s} \otimes \hat\x^T) \w \\
  \hat\x^T W^T W \hat\x &= \w^T (I_{L_s} \otimes \hat\x\hat\x^T) \w \\
  \z^T W \hat\x &= \vec(\z \hat\x^T)^T \w
\end{align}
Here, $\w = \vec(W)$,
$I_{L_s} \in \mathbb{R}^{L_s\times L_s}$, is an identity matrix
and $\otimes$ is the tensor product.
\end{lemma}

\begin{proof}
First, we have
\begin{align}
  W \hat\x
  &=
  \begin{pmatrix}
    \w_1^T \hat\x \\
    \w_1^T \hat\x \\
    \vdots \\
    \w_{L_s}^T \hat\x^T 
  \end{pmatrix}
  =
  \begin{pmatrix}
    \hat\x^T \w_1 \\
    \hat\x^T \w_2 \\
    \vdots \\
    \hat\x^T \w_{L_s}
  \end{pmatrix}
  =
  \begin{pmatrix}
    \hat\x^T &  &  &  \\
     & \hat\x^T &  &  \\
     &    & \ddots &  \\
     &  &  & \hat\x^T
  \end{pmatrix}
  \begin{pmatrix}
      \w_1 \\
      \w_2 \\
      \vdots \\
      \w_{L_s}
    \end{pmatrix}
  \\
  &=
   (I_{L_s} \otimes \hat\x^T) \w.
\end{align}


Using this equation, we have
\begin{align}
  \hat\x^T W^T W \hat\x
  &=
  \w^T
  \begin{pmatrix}
    \hat\x &  &  &  \\
     & \hat\x &  &  \\
     &    & \ddots &  \\
     &  &  & \hat\x
  \end{pmatrix}
  \begin{pmatrix}
    \hat\x^T &  &  &  \\
     & \hat\x^T &  &  \\
     &    & \ddots &  \\
     &  &  & \hat\x^T
  \end{pmatrix}
  \w
  \\
  &=
  \w^T
  \begin{pmatrix}
    \hat\x\hat\x^T &  &  &  \\
     & \hat\x\hat\x^T &  &  \\
     &    & \ddots &  \\
     &  &  & \hat\x\hat\x^T
  \end{pmatrix}
  \w \\
  &=
  \w^T (I_{L_s} \otimes \hat\x\hat\x^T) \w.
\end{align}

Also, we have
\begin{align}
  \z^T W \hat\x
  &=
  \z^T
  \begin{pmatrix}
    \hat\x^T &  &  \\
     & \hat\x^T &  \\
     &    & \ddots \\
  \end{pmatrix}
  \w
  =
  (z_1 \hat\x^T, z_2 \hat\x^T, \ldots)
  \w
  \\
  &=
  \vec
  \begin{pmatrix}
    z_1 \hat\x^T \\
    z_2 \hat\x^T \\
    \vdots
  \end{pmatrix}^T
  \w
  =
  \vec(\z \hat\x^T)^T \w
  =
  \w^T \vec(\z \hat\x^T).
\end{align}

\end{proof}

For later use, we also define
\begin{align}
  U(\hat\x) 
  = 
  (I_{L_s} \otimes \hat\x\hat\x^T).
\end{align}

\subsection{Rewriting terms with vec operator}


In this subsection, we rewrite the $L_2$ term in the objective function using the lemma below.

\begin{lemma}
The $L_2$ constraint term of MMDTL2 can be written as 
\begin{align}
  \frac{1}{2} 
  \sum_{m=1}^M \sum_{n=1}^N s_{nm} \| W \hat\x_m^t - \x_n^s \|_2^2
  &= \frac{1}{2}  (\w^T U \w - 2 \q^T \w + s),
\end{align}
where $U$, $\q$, and $s$ are given in the proof below.
\end{lemma}

\begin{proof}

A single $L_2$ term can be rewritten with lemma \ref{lemma:multiplication_with_vec} as 
\begin{align}
  \| W \hat\x_m^t - \x_n^s \|_2^2
  &= (W \hat\x_m^t - \x_n^s)^T (W \hat\x_m^t - \x_n^s)
  \\
  &= (\hat\x_m^t)^T W^T W \hat\x_m^t - 2 (\x_n^s)^T W \hat\x_m^t + \| \x_n^s \|_2^2
  \\
  &= \w^T U( \hat\x_m^t) \w - 2 \vec(\x_n^s (\hat\x_m^t)^T)^T \w + \| \x_n^s \|_2^2
  \\
  &= \w^T U_m \w - 2 \q_{nm}^T \w + \| \x_n^s \|_2^2,
\end{align}
where 
\begin{align}
  U_m &= U( \hat\x_m^t) = (I_{L_s} \otimes \hat\x_m^t (\hat\x_m^t)^T)
\end{align}
and
\begin{align}
  \q_{nm} &= \vec(\x_n^s (\hat\x_m^t)^T).
\end{align}

By summing the terms with weights, we have
\begin{align}
  &
  \frac{1}{2}   \sum_{m=1}^M \sum_{n=1}^N s_{nm}
  (\w^T U_m \w - 2 \q_{nm}^T \w + \| \x_n^s \|_2^2)
  \\
  =&
  \frac{1}{2}  
   \w \left(  \sum_{m=1}^M \sum_{n=1}^N s_{nm} U_m \right) \w \notag \\
  & -  \left( \sum_{m=1}^M \sum_{n=1}^N s_{nm} \q_{nm}^T \right) \w 
    + \frac{1}{2}   \sum_{m=1}^M \sum_{n=1}^N s_{nm} \| \x_n^s \|_2^2
  \\
  =& \frac{1}{2}  (\w^T U \w - 2 \q^T \w + s).
\end{align}

Here, $U$, $\q$, and $s$ are the corresponding factors;
we further rewrite them into the simpler forms shown below. 
\begin{align}
  U
  &= \sum_{m=1}^M \sum_{n=1}^N s_{nm} U_m 
  = \sum_{m=1}^M s_m U_m \\
  &= \sum_{m=1}^M s_m (I_{L_s} \otimes \hat\x_m^t (\hat\x_m^t)^T)
  = (I_{L_s} \otimes \sum_{m=1}^M s_m \hat\x_m^t (\hat\x_m^t)^T) \\
  &= (I_{L_s} \otimes X^t S_M (X^t)^T )
  \\
  \q
  &= \sum_{m=1}^M \sum_{n=1}^N s_{nm} \q_{nm}
  = \sum_{m=1}^M \sum_{n=1}^N s_{nm} \vec(\x_n^s (\hat\x_m^t)^T) \\
  &=  \vec( \sum_{m=1}^M \sum_{n=1}^N s_{nm} \x_n^s (\hat\x_m^t)^T) \\
  &=  \vec( X^s S (X^t)^T)
  \\
  s
  &= \sum_{m=1}^M \sum_{n=1}^N s_{nm} \| \x_n^s \|_2^2
  = \sum_{n=1}^N s_{n} \| \x_n^s \|_2^2
\end{align}
Here, we use data matrices
\begin{align}
  X^s
  &= (\x_1^s, \x_2^s, \ldots, \x_N^s)
  \in \mathbb{R}^{L_s \times N}
\end{align}
and
\begin{align}
  X^t
  &= (\hat\x_1^t, \hat\x_2^t, \ldots, \hat\x_M^t)
  \in \mathbb{R}^{(L_t+1) \times M}
\end{align}
and weights 
\begin{align}
  S
  &= 
  \begin{pmatrix}
  s_{11} & \cdots & s_{1M} \\
  \vdots & & \vdots \\
  s_{N1} & \cdots & s_{NM}
  \end{pmatrix},
  \\
  s_m &= \sum_{n=1}^N s_{nm},
  \quad
  s_n = \sum_{m=1}^M s_{nm}, \notag \\
  S_M
  &= \mathrm{diag}(s_1, \ldots, s_m, \ldots, s_M).
\end{align}

\end{proof}


Next, we rewrite the conditions in the problem as shown below.

\begin{lemma}
The condition in problem \ref{prob:estimation_w} can be written as
\begin{align}
  y_{km}^t (\bphi_{km}^T \w + b_k) -1 + \xi_{km}^t  & \ge 0,
\end{align}
where $\bphi_{km} = \vec(\th_k (\hat\x_m^t)^T )
$.
\end{lemma}

\begin{proof}
\begin{align}
y_{km}^t \hat\th_k^T 
  \begin{pmatrix}
  W \hat\x_m^t \\
  1
  \end{pmatrix}
   -1 + \xi_{km}^t & \ge 0\\
  y_{km}^t (\th_k^T W \hat\x_m^t + b_k) -1 + \xi_{km}^t  & \ge 0 \\
  y_{km}^t (\vec(\th_k (\hat\x_m^t)^T )^T \w + b_k) -1 + \xi_{km}^t  & \ge 0 \\
  y_{km}^t (\bphi_{km}^T \w + b_k) -1 + \xi_{km}^t  & \ge 0
\end{align}

\end{proof}

\subsection{Primal QP problem}

In this subsection, we write the problem in the form of a canonical QP problem.

\begin{lemma}

Problem \ref{prob:estimation_w} can be written as
\begin{align}
  & \min_{\w, \bxi }
  \frac{1}{2} (\w^T, \bxi^T) 
  \begin{pmatrix}
  V & 0 \\
  0 & 0
  \end{pmatrix}
  \begin{pmatrix}
  \w \\ \bxi
  \end{pmatrix}
  +  (-\q^T, c_T \bm{1}_{KM}^T) \begin{pmatrix}
  \w \\ \bxi
  \end{pmatrix} + \frac{1}{2}  s
  \label{eq:primal_QP_form}
    \\
  & \text{s.t.} \notag
  \\
  &
  \begin{pmatrix}
  0 & I_{KM} \\
  Y \Phi^T & I_{KM}
  \end{pmatrix}
  \begin{pmatrix}
    \w \\ \bxi
  \end{pmatrix} 
   \ge  
  \begin{pmatrix}
    \bm{0} \\
    \bm{1}_{KM} -  Y\bb    
  \end{pmatrix},
\end{align}
where variables are defined in the proof below.
\end{lemma}

\begin{proof}

First, we define two matrices $V \in \mathbb{R}^{L_s (L_t+1) \times L_s (L_t+1)}$ and $A \in \mathbb{R}^{(L_t+1) \times (L_t+1)}$ as follows:
\begin{align}
  V
  &= c_f I_{L_s(L_t+1)} +  U \\
  &= c_f I_{L_s(L_t+1)} + (I_{L_s} \otimes X^t S_M (X^t)^T ) \\
  &= I_{L_s} \otimes (c_f I_{L_t+1} + X^t S_M (X^t)^T) \\
  &= I_{L_s} \otimes A
  \\
  A
  &= c_f I_{L_t+1} + X^t S_M (X^t)^T
\end{align}

Then, we rewrite the objective function (\ref{eq:estimation_w_cost}) as 
\begin{align}
  &
  \frac{1}{2} c_f \| \w \|_2^2
  +
  c_T \sum_{k=1}^K \sum_{m=1}^M \xi_{km}^t
  + \frac{1}{2}  (\w^T U \w - 2 \q^T \w + s)
  \\
  &=
  \frac{1}{2} \w^T V \w -  \q^T \w + \frac{1}{2}  s
  + c_T \bm{1}_{KM}^T \bxi
  \\
  &=
  \frac{1}{2} (\w^T, \bxi^T) 
  \begin{pmatrix}
  V & 0 \\
  0 & 0
  \end{pmatrix}
  \begin{pmatrix}
  \w \\ \bxi
  \end{pmatrix}
  +  (-\q^T, c_T \bm{1}_{KM}^T) \begin{pmatrix}
  \w \\ \bxi
  \end{pmatrix} + \frac{1}{2}  s,
\end{align}
where
\begin{align}
  \bxi &= (\xi^t_{11}, \xi^t_{12}, \ldots, \xi^t_{1M}, \xi^t_{21}, \ldots, \xi^t_{KM})^T.
\end{align}

To rewrite conditions 
(\ref{eq:estimation_w_condition1}) and 
(\ref{eq:estimation_w_condition2}),
we turn these constraints into vector form with a generalized inequality.
The first constraint, i.e., (\ref{eq:estimation_w_condition1}), can be written as follows:
\begin{align}
  \bxi &\ge \bm{0}
  \\
  \begin{pmatrix}
  0 & I_{KM} \\
  \end{pmatrix}
  \begin{pmatrix}
  \w \\ \bxi
  \end{pmatrix}
   &\ge \bm{0}. \label{eq:(63)}
\end{align}
The second constraint, i.e., (\ref{eq:estimation_w_condition2}), is
\begin{align}
Y (\Phi^T\w + \bb )- \bm{1}_{KM} + \bxi & \ge 0 \\
(Y \Phi^T, I_{KM})
\begin{pmatrix}
  \w \\ \bxi
\end{pmatrix} & \ge  \bm{1}_{KM} -  Y\bb, \label{eq:(65)}
\end{align}
where $\bm{1}_{KM}$ is a vector of $KM$ ones and 
\begin{align}
  \b &= (b_1, b_2, \ldots, b_K)^T
  \\
   \bb  &= (\b \otimes \bm{1}_M) 
    = (b_1, b_1, \ldots, b_1, b_2, \ldots, b_K)^T
  \\
  \Phi &= (\bphi_{11}, \bphi_{12}, \ldots, \bphi_{1M}, \bphi_{21}, \ldots, \bphi_{KM})
  \\
  Y &= \mathrm{diag}(y^t_{11}, y^t_{12}, \ldots, y^t_{1M}, y^t_{21}, \ldots, y^t_{KM}).
\end{align}

By combining these inequalities (\ref{eq:(63)}) and (\ref{eq:(65)}),
we have the conditions in a single form, as claimed.

\end{proof}

This primal QP problem involves very large matrices that are impractical to compute.
More precisely, $V$ is a matrix of size $L_s (L_t+1) \times L_s (L_t + 1)$, which can be very large when the dimensions of features ($L_s$ and $L_t$) are large.
In the next section, we therefore derive the dual form of the problem, which we expect to be less expensive to compute.

\section{Dual problem}
\label{sec:Dual problem}

In this section, we derive the dual of the problem.

\subsection{Lagrangian}

\begin{lemma}[Lagrangian]
The Lagrangian of problem (\ref{eq:primal_QP_form}) is given by
\begin{align}
  L =& 
    -\frac{1}{2} \a^T Y \Phi^T V^{-1} \Phi Y \a \notag \\
    & +(\bm{1}^T - \bb^T Y - \q^T V^{-1} \Phi Y ) \a
   -\frac{1}{2} \q^T V^{-1} \q
   + \frac{1}{2} s.
\end{align}

\end{lemma}

\begin{proof}

The Lagrangian of problem (\ref{eq:primal_QP_form}) is given by
\begin{align}
  L &=
  \frac{1}{2} \w^T V \w -  \q^T \w + \frac{1}{2}  s
  + c_T \sum_{k=1}^K \sum_{m=1}^M \xi_{km}^t
  \notag \\ & \phantom{=}
   - \sum_{k=1}^K \sum_{m=1}^M \mu_{km} \xi_{km}^t
   - \sum_{k=1}^K \sum_{m=1}^M a_{km} (y_{km}^t (\bphi_{km}^T \w +b_k) -1 + \xi_{km}^t),
\end{align}
where $a_{km} \ge 0$ and $\mu_{km} \ge 0$ are Lagrange multipliers.

To simplify the derivation, we convert it into vector form
\begin{align}
  L =&
  \frac{1}{2} \w^T V \w -  \q^T \w + \frac{1}{2} s + c_T \bm{1}_{KM}^T \bxi
   - \bmu^T \bxi \notag \\
   & - \a^T (Y (\Phi^T\w + \bb ) - \bm{1}_{KM} + \bxi)
   \\
   =&
  \frac{1}{2} \w^T V \w
   - (\q + \Phi Y \a)^T \w 
   + (c_T \bm{1}_{KM} - \bmu - \a)^T \bxi \notag \\
   & + \a^T (\bm{1}_{KM} - Y \bb )
   + \frac{1}{2} s,
\end{align}
where
\begin{align}
  \a &=  (a_{11}, a_{12}, a_{1M}, a_{21}, \ldots, a_{KM})^T 
\end{align}
and
\begin{align}
  \bmu &= (\mu_{11}, \mu_{12}, \mu_{1M}, \mu_{21}, \ldots, \mu_{KM})^T,
\end{align}
with $\a \ge 0$ and $\bmu \ge 0$.


Next, we take the derivatives of the Lagrangian as follows.
For $\w$, we have
\begin{align}
  \frac{\partial L}{\partial \w}
  =
  V \w - (\q + \Phi Y \a)
   &= \bm{0}
\\
  V \w &= \q + \Phi Y \a
  \\
  \w &= V^{-1} (\q + \Phi Y \a).
\end{align}

For $\bxi$, we have
\begin{align}
  \frac{\partial L}{\partial  \bxi}
  =
  c_T \bm{1} - \bmu - \a
   &= \bm{0}
  \\
  c_T \bm{1} - \a = \bmu & \ge 0
  \\
  c_T \bm{1} - \a & \ge 0
  \\
  c_T \bm{1} & \ge \a \ge 0
\end{align}
for $\bmu \ge 0$ and $\a \ge 0$.



By incorporating $\w$ and $c_T$ into the Lagrangian
and using $V^T = V$, we have
\begin{align}
  L
  =&
    \frac{1}{2} \w^T V \w
   - (\q + \Phi Y \a)^T \w 
   + (c_T \bm{1} - \bmu - \a)^T \bxi \notag \\
   & + \a^T (\bm{1} - Y\bb )
   + \frac{1}{2} s
  \\
  =&
    \frac{1}{2} (V^{-1} (\q + \Phi Y \a))^T V (V^{-1} (\q + \Phi Y \a)) 
    \notag \\ & \phantom{=}
   - (\q + \Phi Y \a)^T (V^{-1} (\q + \Phi Y \a)) 
   + \a^T (\bm{1} - Y\bb )
   + \frac{1}{2} s 
  \\
  =&
    -\frac{1}{2} (\q + \Phi Y \a)^T V^{-1}  (\q + \Phi Y \a)
   + \a^T (\bm{1} - Y\bb )
   + \frac{1}{2} s 
  \\
  =&
    -\frac{1}{2} 
    (
    \q^T V^{-1} \q
    + 2 \q^T V^{-1} \Phi Y \a
    + \a^T Y \Phi^T V^{-1} \Phi Y \a
    ) \notag \\
   & + \a^T (\bm{1} - Y\bb )
   + \frac{1}{2} s 
  \\
  =&
    -\frac{1}{2} \a^T Y \Phi^T V^{-1} \Phi Y \a
    +(\bm{1}^T - \bb^T Y - \q^T V^{-1} \Phi Y ) \a \notag \\
   & -\frac{1}{2} \q^T V^{-1} \q
   + \frac{1}{2} s.
\end{align}

\end{proof}

This is indeed the dual form, but it still involves a large matrix $V$.
In the next subsection, by utilizing the structure of $V$, we will write $L$ in such a way
that it involves only smaller matrices.

\subsection{Lagrangian with a compact form}

To remove the large matrix $V$, 
we use the structure of $V$ and rewrite terms involving $V$
($\phi^T V^{-1} \phi$ and $\q^T V^{-1} \Phi$).


First, the inverses of $V$ and $A$ are as follows:
\begin{align}
  V^{-1}
  & = 
  (I_{L_s} \otimes A)^{-1}
  = 
  I_{L_s} \otimes A^{-1}
\\ 
  A^{-1}
  &=
  (c_f I_{L_t + 1} + X^t S_M (X^t)^T)^{-1}
  \\
  &=
  \frac{1}{c_f} I_{L_t + 1} - \frac{1}{c_f^2} X^t (S_M^{-1} + \frac{1}{c_f} (X^t)^T X^t)^{-1} (X^t)^T. 
\end{align}
Note that the second form of $A^{-1}$ is obtained using Woodbury's formula
only if $S_M$ is non-singular and $c_f \neq 0$; this is usually the case, because diagonal elements of $S_M$ are sums of (non-negative) weights.

\begin{lemma}
Given $V$ of the structure above and vectors
$\a, \c \in \mathbb{R}^{L_s}$ and $\b, \d \in \mathbb{R}^{L_t}$,
we have
\begin{align}
  \vec(\a \b^T)^T
  V^{-1}
  \vec(\c \d^T)
  =
  (\a^T \c) \b^T A^{-1} \d.
\end{align}
\end{lemma}
\begin{proof}
\begin{align}
  &
  \vec(\a \b^T)^T
  V^{-1}
  \vec(\c \d^T)
  \\
  &=
  (a_1 \b^T, a_2 \b^T, \ldots)
  \begin{pmatrix}
    A^{-1} &  &    \\
     & A^{-1} &    \\
     &    & \ddots   \\
  \end{pmatrix}
  \begin{pmatrix}
    c_1 \d \\
    c_2 \d \\
    \vdots
  \end{pmatrix}
  \\
  &=
  (a_1 \b^T, a_2 \b^T, \ldots)
  \begin{pmatrix}
    c_1 A^{-1} \d \\
    c_2 A^{-1} \d \\
    \vdots
  \end{pmatrix}
  \\
  &=
  \sum_i a_i c_i \b^T A^{-1} \d
  =
  (\a^T \c) \b^T A^{-1} \d
\end{align}
\end{proof}


Next, we explore $\phi^T V^{-1} \phi$ via the lemma below.

\begin{lemma}
Given matrix
$\Phi \in \mathbb{R}^{L_s (L_t+1) \times KM} $,
we have
\begin{align}
  \Phi^T V^{-1} \Phi = (\Theta^T \Theta) \otimes G,
\end{align}
where $\Theta = (\th_1, \th_2, \ldots, \th_K) 
\in \mathbb{R}^{Ls\times K}$
and $G \in \mathbb{R}^{M \times M}$, the latter given in the proof below.
\end{lemma}

\begin{proof}

Using the above lemma, we have
\begin{align}
  \bphi_{km}^T V^{-1} \bphi_{k'm'}
  &=
  \vec(\th_k (\hat\x_m^t)^T )^T
  V^{-1}
  \vec(\th_{k'} (\hat\x_{m'}^t)^T )
  \\
  &=
  (\th_k^T \th_{k'}) (\hat\x_m^t)^T A^{-1} \hat\x_{m'}^t.
\end{align}

By stacking the above equation for $m=1,\ldots,M$, we have
\begin{align}
  &
  \begin{pmatrix}
    \bphi_{k1}^T \\
    \vdots \\
    \bphi_{kM}^T
  \end{pmatrix}
  V^{-1}
  \begin{pmatrix}
    \bphi_{k'1}
    \cdots
    \bphi_{k'M}
\end{pmatrix}
\\
&=
(\th_k^T \th_{k'})
\begin{pmatrix}
  (\hat\x_1^t)^T A^{-1} \hat\x_{1}^t & \cdots & (\hat\x_1^t)^T A^{-1} \hat\x_{M}^t \\
  \vdots   &    & \vdots   \\
  (\hat\x_M^t)^T A^{-1} \hat\x_{1}^t & \cdots & (\hat\x_M^t)^T A^{-1} \hat\x_{M}^t \\
\end{pmatrix}
\\
&=
(\th_k^T \th_{k'})
(X^t)^T A^{-1} X^t
\\
&=
(\th_k^T \th_{k'}) G,
\end{align}
where $G = (X^t)^T A^{-1} X^t$.

Finally, 
by stacking the above equation for $k=1,\ldots,K$,
we obtain compact form
\begin{align}
  \Phi^T V^{-1} \Phi
  &=
  \begin{pmatrix}
    \bphi_{11}^T \\
    \vdots \\
    \bphi_{1M}^T \\
    \bphi_{21}^T \\
    \vdots \\
    \bphi_{KM}^T
  \end{pmatrix}
  V^{-1}
  \begin{pmatrix}
    \bphi_{11},
    \cdots
    \bphi_{1M},
    \bphi_{21},
    \cdots
    \bphi_{KM}
\end{pmatrix}
\\
&=
\begin{pmatrix}
  (\th_1^T \th_{1}) G & (\th_1^T \th_{2}) G & \cdots & (\th_1^T \th_{K}) G  \\
  (\th_2^T \th_{1}) G & (\th_2^T \th_{2}) G & \cdots & (\th_2^T \th_{K}) G  \\
  \vdots &  \vdots  & \ddots  & \vdots \\
  (\th_K^T \th_{1}) G & (\th_K^T \th_{2}) G & \cdots & (\th_K^T \th_{K}) G
\end{pmatrix}
\\
&=
(\Theta^T \Theta) \otimes G,
\end{align}
where $\Theta = (\th_1, \th_2, \ldots, \th_K) 
$.

\end{proof}

Note that we can rewrite $G$ further as 
\begin{align}
G
=&
(X^t)^T A^{-1} X^t
\\
=&
(X^t)^T  (\frac{1}{c_f} I_{L_t + 1} - \frac{1}{c_f^2} X^t (S_M^{-1} + \frac{1}{c_f} (X^t)^T X^t)^{-1} (X^t)^T) X^t
\\
=&
\frac{1}{c_f}(X^t)^T X^t \notag \\ & - \frac{1}{c_f^2}(X^t)^T X^t (S_M^{-1} + \frac{1}{c_f}(X^t)^T X^t)^{-1} (X^t)^T X^t
\\
=&
\frac{1}{c_f} K^t - \frac{1}{c_f^2} K^t (S_M^{-1} + \frac{1}{c_f} K^t)^{-1}  K^t,
\end{align}
where
$K^t = (X^t)^T X^t 
$ is a kernel matrix.
If $(K^t)^{-1}$ exists, we obtain
\begin{align}
G
&=
(c_f (K^t)^{-1} + S_M)^{-1} 
\end{align}
by applying the Woodbury formula.

In summary, $G \in \mathbb{R}^{M \times M}$ is 
\begin{align}
G &= 
\begin{cases}
(c_f (K^t)^{-1} + S_M)^{-1}, & \text{if $(K^t)^{-1}$ exists,}\\
\frac{1}{c_f} K^t - \frac{1}{c_f^2} K^t (S_M^{-1} + \frac{1}{c_f} K^t)^{-1}  K^t
, & \text{if $(S_M^{-1} + \frac{1}{c_f} K^t)^{-1}$ exists,}\\
(X^t)^T A^{-1} X^t,  & \text{otherwise,}
\end{cases}
\end{align}
depending on the existence of 
the inverse of kernel matrix $K^t \in \mathbb{R}^{M \times M}$.
It exists when the dimension of the column space of $X^t$ is $M$,
i.e., when the target samples are linearly independent.
If not, the second option can be used,
i.e., the inverse of $S_M^{-1} + \frac{1}{c_f} K^t$,
which can be interpreted as the regularization of kernel $K^t$ with diagonal weight matrix $S_M^{-1}$.


In the next lemma, we rewrite $\q^T V^{-1} \Phi$.

\begin{lemma}
Given matrix
$\Phi \in \mathbb{R}^{L_s (L_t+1) \times KM} $
and vector 
$\q \in \mathbb{R}^{L_s (L_t+1)}$,
we have
\begin{align}
\q^T V^{-1}  \Phi =
\vec \left(
\Theta^T
X^s S G
\right)^T.
\end{align}
\end{lemma}

\begin{proof}

Using the above lemma in a similar way as with the previous lemma, we have
\begin{align}
  \q_{nm}^T V^{-1} \phi_{k'm'}
  &=
  \vec(\x_n^s (\hat\x_m^t)^T)^T
  V^{-1}
  \vec(\th_{k'} (\hat\x_{m'}^t)^T )
  \\
  &=
  ((\x_n^s)^T\th_{k'}) (\hat\x_m^t)^T A^{-1} \hat\x_{m'}^t
  \\
  &=
  (\th_{k'}^T \x_n^s) (\hat\x_m^t)^T A^{-1} \hat\x_{m'}^t.
\end{align}

By adding and stacking the equation, we have
\begin{align}
&
\left(   \sum_{m=1}^M \sum_{n=1}^N s_{nm} \q_{nm}^T \right)
V^{-1}
  \begin{pmatrix}
    \bphi_{k'1},
    \cdots
    \bphi_{k'M},
\end{pmatrix}
\\
&=
  \sum_{m=1}^M \sum_{n=1}^N s_{nm}
  (\th_{k'}^T \x_n^s) 
  (
  (\hat\x_m^t)^T A^{-1} \hat\x_{1}^t,
  \ldots,
  (\hat\x_m^t)^T A^{-1} \hat\x_{M}^t
  )
  \\
  &=
  \sum_{m=1}^M \sum_{n=1}^N s_{nm}
  (\th_{k'}^T \x_n^s) 
  (
  (\hat\x_m^t)^T
   A^{-1} 
   (\hat\x_{1}^t,
    \ldots,
    \hat\x_{M}^t
    )
  )
\\
  &=
  \sum_{m=1}^M \sum_{n=1}^N s_{nm}
  (\th_{k'}^T \x_n^s) 
  (\x_m^t)^T
   A^{-1} 
   X^t
\\
  &=
  \th_{k'}^T 
\sum_{m=1}^M \sum_{n=1}^N s_{nm}
  \x_n^s
  (\hat\x_m^t)^T
   A^{-1} 
   X^t
\\
  &=
  \th_{k'}^T 
\left(
\sum_{m=1}^M \sum_{n=1}^N s_{nm}
  \x_n^s
  (\hat\x_m^t)^T
\right)
   A^{-1} 
   X^t
   \\
  &=
  \th_{k'}^T 
  X^s S (X^t)^T A^{-1} X^t
  \\
  &=
  \th_{k'}^T 
  X^s S G.
\end{align}

Finally, by stacking the equation, we have
\begin{align}
\q^T V^{-1}  \Phi
&=
(
 \th_{1}^T X^{st} A^{-1} X^t,
\ldots,
 \th_{K}^T X^{st} A^{-1} X^t
)
\\
&=
(
 \th_{1}^T X^s S G,
\ldots,
 \th_{K}^T X^s S G
)
\\
&=
\vec \left(
\begin{pmatrix}
\th_{1}^T \\
\vdots \\
\th_{K}^T
\end{pmatrix}
X^s S G
\right)^T
\\
&=
\vec \left(
\Theta^T
X^s S G
\right)^T.
\end{align}

\end{proof}


We now have the final form of the dual and present it in the corollary below.

\begin{corollary}
The dual form of the original primal problem is given by
\begin{align}
\max_{\a} &
 \frac{-1}{2}  \bm{a}^T  Y^T ( (\Theta^T \Theta) \otimes G ) Y \bm{a} \notag \\ &
   + (\bm{1}^T - (Y\bb )^T - \vec \left(\Theta^T X^s S G \right)^T
   Y) \bm{a}
 \\
& \text{s.t.} \quad c_T \bm{1} \ge \a \ge 0.
\end{align}
\end{corollary}

\begin{proof}

According to the lemmas derived above,
we can write the Lagrangian as 
\begin{align}
  L =&
  \frac{-1}{2}  \bm{a}^T  Y^T  \Phi^T V^{-1} \Phi Y \bm{a}
  + (\bm{1}^T- \q^T V^{-1}  \Phi Y) \bm{a} \notag \\ &
  + \frac{-1}{2}  \q^T V^{-1} \q
  + \frac{1}{2} s
  \\
  =&
  \frac{-1}{2}  \bm{a}^T  Y^T ( (\Theta^T \Theta) \otimes G ) Y \bm{a} \notag \\ &
   + (\bm{1}^T - (Y\bb )^T - \vec \left(\Theta^T X^s S G \right)^T
   Y) \bm{a} 
  \notag \\ & 
  + \frac{-1}{2}  \q^T V^{-1} \q
  + \frac{1}{2} s.
\end{align}
By omitting the last two terms, which do not involve $\a$, we have the dual problem as claimed.
\end{proof}

Note that this dual form involves matrices of size at most $KM \times KM$,
which is reasonable when the number of categories $K$ 
and the number of samples in the target domain $M$ are both small.

If all $s_{nm} = 0$, then the problem reduces to MMDT, i.e.,
\begin{align}
  L  &=
  \frac{-1}{2}  \bm{a}^T  Y^T ( (\Theta^T \Theta) \otimes G) Y \bm{a} 
   + (\bm{1}^T - (Y\bb )^T ) \bm{a},
\end{align}
where $G = (X^t)^T X^t$, since $A=I$.





\section{Retrieving the primal solution}

After solving the dual problem with a QP solver,
we need to convert the dual solution $\a$ to $\w$ and $\b$ by
\begin{align}
  \w &= V^{-1} (\q + \Phi Y \a),
\end{align}
then finally to $W$.

Here, we again face the problem of large matrix $V$.
We, therefore, derive the core parts $V^{-1} \q$ and $V^{-1} \Phi$
as shown in the lemmas below.


\begin{lemma}
Given $V$ of the structure above and vectors $\c \in \mathbb{R}^{L_s}$
and $\d \in \mathbb{R}^{L_t}$, we have
\begin{align}
  V^{-1}  \vec(\c \d^T)
  =
  \c \otimes (A^{-1} \d).
\end{align}
\end{lemma}
\begin{proof}

\begin{align}
  V^{-1}
  \vec(\c \d^T)
  &=
  \begin{pmatrix}
    A^{-1} &  &    \\
     & A^{-1} &    \\
     &    & \ddots   \\
  \end{pmatrix}
  \begin{pmatrix}
    c_1 \d \\
    c_2 \d \\
    \vdots
  \end{pmatrix}
  \\
  &
  =
  \begin{pmatrix}
    c_1 A^{-1} \d \\
    c_2 A^{-1} \d \\
    \vdots
  \end{pmatrix}
  =
  \c \otimes (A^{-1} \d)
\end{align}
\end{proof}

\begin{lemma}
Given $V$, $\q$, and $\Phi$, we have
\begin{align}
  V^{-1} \q
  &= 
  \sum_{m=1}^M \sum_{n=1}^N s_{nm}  
  \x_n^s 
  \otimes ( A^{-1} \hat\x_m^t)
\end{align}
and
\begin{align}
  V^{-1} \Phi
  &=
  \Theta \otimes (A^{-1} X^t).
\end{align}
\end{lemma}

\begin{proof}

For (136), we have
\begin{align}
V^{-1} \q
&=
V^{-1}
  \left(   \sum_{m=1}^M \sum_{n=1}^N s_{nm} \q_{nm}\right)
  \\
&=
V^{-1}
  \left(   \sum_{m=1}^M \sum_{n=1}^N s_{nm}  \vec(\x_n^s (\hat\x_m^t)^T) \right)
  \\
&=
\sum_{m=1}^M \sum_{n=1}^N s_{nm}  
V^{-1}
  \vec(\x_n^s (\hat\x_m^t)^T)
  \\
&=
\sum_{m=1}^M \sum_{n=1}^N s_{nm}  
  \x_n^s 
  \otimes ( A^{-1} \hat\x_m^t).
\end{align}


Next, we derive (137) by first stacking 
\begin{align}
  V^{-1} \bphi_{k'm'}
  &=
  V^{-1}
  \vec(\th_{k'} (\hat\x_{m'}^t)^T )
  \\
  &=
  \th_{k'} \otimes
  (A^{-1} \hat\x_{m'}^t)
\end{align}
for $m=1,\ldots,M$ to obtain
\begin{align}
&
  V^{-1} 
  \begin{pmatrix}
    \bphi_{k'1}
    \cdots
    \bphi_{k'M}
   \end{pmatrix}
  \\
  &=
  (
  \th_{k'} \otimes (A^{-1} \hat\x_{1}^t),
  \ldots,
  \th_{k'} \otimes (A^{-1} \hat\x_{M}^t)
  )
  \\
  &=
  \th_{k'} \otimes 
  (
  A^{-1} \hat\x_{1}^t,
  \ldots,
  A^{-1} \hat\x_{M}^t
  ) 
  \\
  &=
  \th_{k'} \otimes 
  (A^{-1} 
  (
  \hat\x_{1}^t,
  \ldots,
  \hat\x_{M}^t
  ))
  \\
  &=
  \th_{k'} \otimes (A^{-1} X^t).
\end{align}

Then, we further stack the above equation for $k=1,\ldots,K$
to obtain
\begin{align}
  V^{-1} \Phi
  &=
  V^{-1} 
  \begin{pmatrix}
    \bphi_{11},
    \cdots
    \bphi_{1M},
    \bphi_{21},
    \cdots
    \bphi_{KM}
\end{pmatrix}
\\
&=
(
 \th_{1} \otimes (A^{-1} X^t),
 \ldots,
 \th_{K} \otimes (A^{-1} X^t)
)
\\
&=
( \th_{1}, \ldots, \th_{K} ) \otimes (A^{-1} X^t) 
\\
&=
\Theta \otimes (A^{-1} X^t). 
\end{align}

\end{proof}

Finally, we show primal solution $W$ directly,
i.e., avoiding conversions from $\a$ to $\w$, and then to $W$.
In the corollary below, we construct matrix $W$ from $\a$
with much less computational costs.

\begin{corollary}
The solution to the primal problem is given by
\begin{align}
  W &=
  \left( X^s S
  +
  \Theta ( \Upsilon \odot \Lambda )^T  \right)
  (X^t)^T A^{-1},
\end{align}
where $\odot$ is element-wise multiplication
and variables are given in the proof below.
\end{corollary}

\begin{proof}

We first use the lemmas above to obtain
\begin{align}
  \w &= V^{-1} (\q + \Phi Y \a)
  \\
  &= 
  \left( \sum_{m=1}^M \sum_{n=1}^N s_{nm}  
  \x_n^s \otimes ( A^{-1} \hat\x_m^t)
  \right)
  + 
  \left( \Theta \otimes (A^{-1} X^t) \right) Y \a.
\end{align}

For the $i$th part of $\w$, we have
\begin{align}
  \w_i =&
    \left( \sum_{m=1}^M \sum_{n=1}^N s_{nm}  
  x_{n,i}^s ( A^{-1} \hat\x_m^t)
  \right) \notag \\ &
  + 
  \left( (\theta_{1,i}, \ldots, \theta_{K,i}) \otimes (A^{-1} X^t) \right) Y \a.
\end{align}
The first term here can be written as 
\begin{align}
  A^{-1} \left( \sum_{m=1}^M \sum_{n=1}^N s_{nm}  
  x_{n,i}^s  \hat\x_m^t
  \right)
  =
  A^{-1} X^t S^T 
  \begin{pmatrix}
  x_{1,i}^s \\
  x_{2,i}^s \\
  \vdots \\
  x_{N,i}^s
  \end{pmatrix}
\end{align}
and the second term as 
\begin{align}
  &
  (\theta_{1,i} A^{-1} X^t, \ldots, \theta_{K,i} A^{-1} X^t) Y \a
  \\
  &=
  \theta_{1,i} A^{-1} X^t Y_1 \a_1 
  + \cdots +
  \theta_{K,i} A^{-1} X^t Y_K \a_K
  \\
  &=
  A^{-1} X^t ( \theta_{1,i} Y_1 \a_1 
  + \cdots +
  \theta_{K,i} Y_K \a_K )
  \\
  &=
  A^{-1} X^t
  (  Y_1 \a_1, \ldots, Y_K \a_K )
  \begin{pmatrix}
  \theta_{1,i} \\
  \vdots \\
  \theta_{K,i}
  \end{pmatrix}
  \\
  &=
  A^{-1} X^t
  ( \Upsilon \odot \Lambda ) 
  \begin{pmatrix}
  \theta_{1,i} \\
  \vdots \\
  \theta_{K,i}
  \end{pmatrix},
\end{align}
where
\begin{align}
\a_k &= \begin{pmatrix}
a_{k1} \\ a_{k2} \\ \vdots \\ a_{kM}
\end{pmatrix}
\\
\Lambda &= (\a_1, \ldots, \a_K) =
 \begin{pmatrix}
a_{11} & \cdots & a_{K1} \\
\vdots & & \vdots \\
a_{1M} & \cdots & a_{KM}
\end{pmatrix}
\\
Y_k &= \mathrm{diag}(y^t_{k1}, \ldots, y^t_{kM})
\\
\Upsilon &= 
\begin{pmatrix}
y^t_{11} & \cdots & y^t_{K1} \\
\vdots & & \vdots \\
y^t_{1M} & \cdots & y^t_{KM}
\end{pmatrix}.
\end{align}

Combining these two terms, we obtain
\begin{align}
\w_i &=
A^{-1} X^t S^T 
  \begin{pmatrix}
  x_{1,i}^s \\
  x_{2,i}^s \\
  \vdots \\
  x_{N,i}^s
  \end{pmatrix}
  +
  A^{-1} X^t
  (  Y_1 \a_1, \ldots, Y_K \a_K )
  \begin{pmatrix}
  \theta_{1,i} \\
  \vdots \\
  \theta_{K,i}
  \end{pmatrix}.
\end{align}

By stacking the $i$th part for $i=1,\ldots,L_s$, we yield the matrix
directly as
\begin{align}
  W^T =& (\w_1, \ldots, \w_{L_s}) =
  A^{-1} X^t S^T (X^s)^T \notag \\ &
  +
  A^{-1} X^t
  (  Y_1 \a_1, \ldots, Y_K \a_K )
  \Theta^T
  \\
  =&
  A^{-1} X^t 
  \left( S^T (X^s)^T
  +
  ( Y_1 \a_1, \ldots, Y_K \a_K ) \Theta^T \right)
  \\
  =&
  A^{-1} X^t 
  \left( S^T (X^s)^T
  +
  ( \Upsilon \odot \Lambda ) \Theta^T \right).
\end{align}

\end{proof}

\section{Kernelization}

In this section, we derive the kernel version of the dual formulation.
The obtained transformation is further rewritten as
\begin{align}
  W
  =&
  \left( X^s S
  +
  \Theta ( \Upsilon \odot \Lambda )^T  \right)
  (X^t)^T A^{-1}
  \\
  =&
  \left( X^s S
  +
  \Theta ( \Upsilon \odot \Lambda )^T  \right) 
  (X^t)^T \notag \\ &
  \left(\frac{1}{c_f} I_{L_t + 1} - \frac{1}{c_f^2} X^t (S_M^{-1} + \frac{1}{c_f} (X^t)^T X^t)^{-1} (X^t)^T \right)
  \\
  =&
  \left( X^s S
  +
  \Theta ( \Upsilon \odot \Lambda )^T  \right)  \notag \\ &
  \left(\frac{1}{c_f} (X^t)^T - \frac{1}{c_f^2} (X^t)^T X^t (S_M^{-1} + \frac{1}{c_f} (X^t)^T X^t)^{-1} (X^t)^T \right)
  \\
  =&
  \left( X^s S
  +
  \Theta ( \Upsilon \odot \Lambda )^T  \right)  \notag \\ &
  \left(\frac{1}{c_f} I_{M} - \frac{1}{c_f^2} K^t (S_M^{-1} + \frac{1}{c_f} K^t)^{-1} \right) (X^t)^T.
\end{align}
We apply $W$ to target sample $\x^t$ by multiplying it from the left, i.e., $W \hat{\x}^t$;
Therefore, all computations with target samples are inner products,
which means we can use kernels to replace the inner products.

In the dual form, we write matrix $G$ with the kernel version using kernel matrix $K^t$ as
\begin{align}
  K^t
  &=
  \begin{pmatrix}
  k(\x^t_1, \x^t_1) & \cdots & k(\x^t_1, \x^t_M) \\
  \vdots & \ddots & \vdots\\
  k(\x^t_M, \x^t_1) & \cdots & k(\x^t_M, \x^t_M)
  \end{pmatrix},
\end{align}
where $k()$ is a kernel function.
To transform target sample $\x^t$ with $W$, we have
\begin{align}
  W \hat{\x}^t
  =&
  \left( X^s S
  +
  \Theta ( \Upsilon \odot \Lambda )^T  \right) \notag \\ &
  \left(\frac{1}{c_f} I_{M} - \frac{1}{c_f^2} K^t (S_M^{-1} + \frac{1}{c_f} K^t)^{-1} \right) 
  \begin{pmatrix}
  k(\x^t_1, \x^t) \\
  \vdots\\
  k(\x^t_M, \x^t)
  \end{pmatrix}.
\end{align}

Note that the nonlinearity introduced by this kernelization appears only
in the transformation part; target samples are transformed nonlinearly to the source domain.
Only linear SVMs in the source domain can be used here
because the primal solutions (i.e., hyperplane parameters) of the source-domain SVMs are explicitly used in the estimation of $W$.
Target samples are therefore linearly classified in the source domain after being nonlinearly transformed from the target domain.

\section*{Acknowledgment}

Part of this work was supported by JSPS KAKENHI Grant Numbers JP26280015 and JP14J00223,
and by the Research Center for Biomedical Engineering,
and was with the help of a grant by Chugoku Industrial Innovation Center.

\ifCLASSOPTIONcaptionsoff
  \newpage
\fi



\bibliographystyle{IEEEtran}
\end{document}